\documentclass[opre,nonblindrev]{informs3}

\OneAndAHalfSpacedXI 

\usepackage{endnotes}

\usepackage{natbib}
 \bibpunct[, ]{(}{)}{,}{a}{}{,}%
 \def\bibfont{\footnotesize}%

\usepackage{appendix}
\RequirePackage[colorlinks,citecolor=blue,urlcolor=blue]{hyperref}
\RequirePackage{graphicx}
\usepackage[compact]{titlesec}
\usepackage[capitalise]{cleveref}

\crefname{assumption}{Assumption}{Assumptions}
\Crefname{assumption}{Assumption}{Assumptions}

\usepackage{lmodern}
\usepackage{xcolor}
\usepackage{graphicx}
\usepackage{enumerate}
\usepackage{natbib}
\usepackage{url} 
\usepackage{mathtools}
\usepackage{multirow}

\TheoremsNumberedThrough     
\ECRepeatTheorems

\EquationsNumberedThrough    

\usepackage{booktabs} 

\usepackage{algorithmic}
\usepackage[ruled,linesnumbered, boxed]{algorithm2e}

\usepackage[most]{tcolorbox}

\usepackage{amsmath,amsfonts,bm}

\def\eqref#1{Eq.~(\ref{#1})}
\def\Eqref#1{Eq.~(\ref{#1})}

\def\1{\bm{1}}

\def\eps{{\epsilon}}

\def\vzero{{\bm{0}}}

\def\va{{\bm{a}}}
\def\vb{{\bm{b}}}

\def\vd{{\bm{d}}}
\def\ve{{\bm{e}}}

\def\vm{{\bm{m}}}

\def\vr{{\bm{r}}}
\def\vs{{\bm{s}}}

\def\vu{{\bm{u}}}
\def\vv{{\bm{v}}}

\def\vx{{\bm{x}}}

\def\mA{{\bm{A}}}
\def\mB{{\bm{B}}}
\def\mC{{\bm{C}}}
\def\mD{{\bm{D}}}
\def\mE{{\bm{E}}}

\def\mI{{\bm{I}}}

\def\mK{{\bm{K}}}

\def\mP{{\bm{P}}}
\def\mQ{{\bm{Q}}}

\def\mU{{\bm{U}}}
\def\mV{{\bm{V}}}
\def\mW{{\bm{W}}}
\def\mX{{\bm{X}}}
\def\mY{{\bm{Y}}}
\def\mZ{{\bm{Z}}}

\DeclareMathAlphabet{\mathsfit}{\encodingdefault}{\sfdefault}{m}{sl}
\SetMathAlphabet{\mathsfit}{bold}{\encodingdefault}{\sfdefault}{bx}{n}

\def\gA{{\mathcal{A}}}
\def\gB{{\mathcal{B}}}

\def\gD{{\mathcal{D}}}

\def\gF{{\mathcal{F}}}

\def\gM{{\mathcal{M}}}
\def\gN{{\mathcal{N}}}
\def\gO{{\mathcal{O}}}
\def\gP{{\mathcal{P}}}

\def\gS{{\mathcal{S}}}

\def\gZ{{\mathcal{Z}}}

\newcommand{\E}{\mathbb{E}}

\newcommand{\R}{\mathbb{R}}

\DeclareMathOperator{\Tr}{Tr}

\usepackage{xspace}

\newcommand{\pa}[1]{\left({#1}\right)}
\newcommand{\br}[1]{\left[{#1}\right]}
\newcommand{\brk}[1]{\left\{{#1}\right\}}
\newcommand{\norm}[1]{\left\lVert#1\right\rVert}

\newcommand{\abs}[1]{\left\lvert#1\right\rvert}

\def\vzero{{\bf{0}}}

\newcommand{\PAB}{\mP^\pi_{AB}}
\newcommand{\PA}{\mP^\pi_A}
\newcommand{\PB}{\mP^\pi_B}
\newcommand{\muAB}{{\mu^\pi_{AB}}}
\newcommand{\muA}{{\mu^\pi_A}}
\newcommand{\SEP}{\gP_{\textrm{SEP}}}
\newcommand{\PN}{\mP^\pi_{1:N}}

\begin{document}


\RUNAUTHOR{Chen and Peng}

\RUNTITLE{Markov Entanglement}

\TITLE{Multi-agent Markov Entanglement}

\ARTICLEAUTHORS{%
\AUTHOR{Shuze Chen}
\AFF{Graduate School of Business, Columbia University, New York, NY 10027, \EMAIL{shuze.chen@columbia.edu}
 }
\AUTHOR{Tianyi Peng}
\AFF{Graduate School of Business, Columbia University, New York, NY 10027, \EMAIL{tianyi.peng@columbia.edu}}
} 

\ABSTRACT{%

Value decomposition has long been a fundamental technique in multi-agent dynamic programming and reinforcement learning (RL). Specifically, the value function of a global state $(s_1,s_2,\ldots,s_N)$ is often approximated as the sum of local functions: $V(s_1,s_2,\ldots,s_N)\approx\sum_{i=1}^N V_i(s_i)$. This approach traces back to the index policy in restless multi-armed bandit problems and has found various applications in modern RL systems. However, the theoretical justification for why this decomposition works so effectively remains underexplored.

In this paper, we uncover the underlying mathematical structure that enables value decomposition. We demonstrate that a multi-agent Markov decision process (MDP) permits value decomposition \emph{if and only if} its transition matrix is not ``entangled''---a concept analogous to quantum entanglement in quantum physics. Drawing inspiration from how physicists measure quantum entanglement, we introduce how to measure the ``Markov entanglement'' for multi-agent MDPs and show that this measure can be used to bound the decomposition error in general multi-agent MDPs.

Using the concept of Markov entanglement, we prove that a widely-used class of index policies is weakly entangled and enjoys a sublinear $\gO(\sqrt{N})$ scale of decomposition error for $N$-agent systems. Finally, we show how Markov entanglement can be efficiently estimated in practice, providing practitioners with an empirical proxy for the quality of value decomposition.
}%



\maketitle

%


\makeatletter
\def\proof{\@ifnextchar[{\@proof}{\@proof[Proof]}}
\def\@proof[#1]{\begin{trivlist}\item[\hspace*{1em}\hskip\labelsep{\it #1.\enskip}]\ignorespaces}
\makeatother
\def\endproof{\ifmmode\eqno\qed\else\qed\par\fi\end{trivlist}\addvspace{0pt}} 
\def\qed{\hfill\ensuremath{\square}} 


\section{Introduction}
Learning the value function given certain policy, or \emph{policy evaluation}, is one of the most fundamental tasks in RL. Significant attention has been paid to single-agent policy evaluation (\citealt{Sutton18RL, bertsekas1996neuro, john96analysis}). However, when it comes to multi-agent reinforcement learning (MARL), single-agent methodologies typically suffer from \emph{the curse of dimensionality}: the state space of the system scales exponentially with the number of agents. To tackle this problem, one common technique is to decompose the global value function,
\[V(s_1,s_2,\ldots,s_N)\approx\sum_{i=1}^N V_i(s_i)\,,\]
where $V_i$ is some local function that can be learned independently by each agent. It quickly follows that this decomposition greatly reduce the computation complexity from exponential to linear dependency on the number of agents $N$.

The remaining question is whether this decomposition is effective. This is non-trivial due to the coupling of agents---individual agent's action and transition depend on other agents. For example, in a ride-hailing platform, if one driver took the order, then other drivers are not allowed fulfill the same order. As a result, value decomposition may lose information and introduce bias without considering the global constraints. 

In the past several decades, both positive and negative results have been reported. Back to the last century, \cite{Whittle88restless, weber90on} apply Lagrange relaxations to decompose the global value and obtain the well-known Whittle index policy. The Lagrange decomposition idea has also been proved successful in many other important multi-agent tasks such as network revenue management (\citealt{adelman2007dynamic,zhang2009approximate}), resource allocation (\citealt{kadota2016minimizing, santiage23thebest}), and online matching (\citealt{brown2022dynamic, david23on, shar2023weakly, yash24blind}). However, Lagrange decomposition relies on the knowledge of system dynamics and \cite{daniel08relaxation} demonstrates its decomposition error can be arbitrarily bad for general weakly-coupled MDPs. In more recent days, practitioners apply online (deep) reinforcement learning to train a local value function for each individual agent. Albeit little theoretical understanding, this decomposition witnessed great empirical success in practice. For example, ride-hailing platforms conduct policy iteration with global value approximated by the summation of local value functions learned by individual drivers. This practice gives birth to state-of-the-art dispatching policies and have been well recognized by the operations research community, such as DiDi Chuxing (\citealt{qin20ride}\textcolor{blue}{, [Daniel H. Wagner Prize]}) and Lyft (\citealt{xabi24better}\textcolor{blue}{, [Franz Edelman Laureates]}). Intervention policies based on similar value decomposition idea also demonstrate substantial empirical advantage and have been deployed by a behavioral health platform in Kenya (\citealt{baek2023policy}\textcolor{blue}{, [Pierskalla
Award]}). In broader MARL literature, value decomposition serves as one key component of centralized training and decentralized execution (CTDE) paradigm, achieving state-of-the-art performance (\citealt{peter18value,  rashid2020monotonic}). In particular, agents optimize their local value functions and combining them to obtain the global optimal policy, adhering to the individual-global max (IGM) principle. However, recent research has started reflecting on invalidity and potential flaw of this principle in practice (\citealt{hong2022rethinking, dou2022understanding}). 

Despite all these empirical success and failures, there remains little theoretical understanding of whether and how we can decompose the value function in multi-agent Markov systems. 

\subsection{This Paper}
In this paper, we will uncover the underlying mathematical structure that enables/disables value decomposition. Our new theoretical framework quantifies the inter-dependence of agents in multi-agent MDPs and systematically characterizes the effectiveness of value decomposition. For simplicity, we will demonstrate the main results through two-agent MDPs indexed by agent $A$ and $B$. We later extend our results to general $N$-agent MDPs in section~\ref{sec: multi-agent results}. 

We start with a trivial example where two agents are independent, i.e. each following independent MDPs. It's clear that the global value function can be decomposed as the summation of value functions of local MDPs. As two agents are independent, it holds $P^\pi(s_A^\prime,s_B^\prime\mid s_A,s_B)=P^\pi(s_A^\prime\mid s_A) \cdot P^\pi(s_B^\prime\mid s_B)$, or in matrix form, 
\[\PAB=\mP^\pi_A\otimes\mP^\pi_B\,,\]
where $\otimes$ is the tensor product or Kronecker product of matrices. The important question is whether we can extend beyond this trivial case of independent systems.

\paragraph{\textbf{A Sufficient and Necessary Condition}} 
We introduce a new condition called ``Markov Entanglement" to describe the intrinsic structure of the transition dynamics in multi-agent MDPs. 

\begin{center}
\begin{tcolorbox}[colback=blue!5!white,colframe=blue!75!black,title=Markov Entanglement]
  Consider a two-agent MDP with transition $\mP^\pi_{AB}$. If there exists
  \[\PAB = \sum_{j=1}^K x_j\mP_A^{(j)}\otimes \mP_B^{(j)}\,,\]
  then $\PAB$ is separable; otherwise is entangled.
\end{tcolorbox}
\end{center}
Compared with the preceding example of independent subsystems, Markov entanglement offers an intuitive interpretation: a two-agent MDP is separable if it can be expressed as a \emph{linear combination of independent systems}. We then demonstrate,
\[\textrm{separable }\PAB  \Longleftrightarrow \textrm{ decomposable } \mV^\pi_{AB}\,,\]
where $\mV^\pi_{AB}$ is decomposable if there exist local value functions $\mV_A,\mV_B$ such that $V^\pi_{AB}(s_A,s_B)=V_A(s_A)+V_B(s_B)$ for all $(s_A,s_B)$. This result sharply unravels the secret structure of system dynamics governing value decomposition. As a sufficient condition, our finding strictly generalizes the previous independent subsystem example, extending it to scenarios involving interacting and coupled agents. As a necessary condition, we prove that exact value decomposition under any reward kernel requires the system dynamics to be separable. Taken together, this result provides a \textit{complete characterization} of when exact value function decomposition is possible in multi-agent MDPs.

More interestingly, our Markov entanglement condition turns out be a mathematical counterpart of quantum entanglement in quantum physics, whose definition is provided below.
\begin{center}
\begin{tcolorbox}[colback=orange!5!white,colframe=orange!75!black,title=Quantum Entanglement]
  Consider a two-party quantum state $\rho_{AB}$. If there exists
  \[\rho_{AB} = \sum_{j=1}^K x_j\rho_A^{(j)}\otimes \rho_B^{(j)}\,,\quad x_j\geq 0\,,\]
  then $\rho_{AB}$ is separable; otherwise is entangled.
\end{tcolorbox}
\end{center}
The quantum state is represented by a \emph{density matrix}, a positive semi-definite matrix with unit trace, analogous to transition matrix in the Markov world. The concept of quantum entanglement describes the inter-dependence of particles in a quantum system, while Markov entanglement describes that of agents in a Markov system. 

Finally, we introduce several novel proof techniques concerning the sufficient and necessary condition, including an ``absorbing'' technique for separable transition matrices and a novel characterization of the linear space spanned by tensor products of transition matrices. We believe these techniques hold independent interest for the broader RL community.

\paragraph{\textbf{Decomposition Error in General Multi-agent MDPs}}
Despite the precise characterization of Markov entanglement and exact value decomposition, general multi-agent MDPs can exhibit arbitrary complexity, with agents intricately entangled. This raises a critical question: \emph{can value decomposition serve as a meaningful approximation in such scenarios?} To address this, we introduce a mathematical quantification to measure the Markov entanglement in general multi-agent MDPs,
\[E(\PAB)\coloneqq\min_{\mP\in\gP_{\rm{SEP}}} d(\PAB, \mP)\,,\]
where $\gP_{\rm{SEP}}$ is the set of all separable transition matrices and $d(\cdot,\cdot)$ is some distance measure. In other words, the degree of Markov entanglement is determined by its distance to the closest separable transition matrix. This concept can also find its counterpart in quantum physics, with the measure of quantum entanglement defined as 
\[E(\rho_{AB})\coloneqq \min_{\rho\in\rho_{\rm{SEP}}} d(\rho_{AB}, \rho)\,,\]
where $\rho_{\rm{SEP}}$ is the set of all separable quantum states. In quantum physics, various distance measures have been designed for density matrices and capture different physical interpretations (\citealt{nielsen2010quantum}). In the Markov world, we analogously design distance measures for transition matrices and relate them to the value decomposition error, 
\[\Big\|\textrm{decomposition error of }\mV^\pi_{AB} \Big\| = \gO\Big(E(\PAB)\Big)\,.\]
where $\norm{\cdot}$ depends on the distance we use to measure Markov entanglement. 
We explore diverse distance measures including the well-known total variation distance and its stationary distribution weighted variant. We also design a novel agent-wise distance incorporating the multi-agent structure, which may be of independent interest to the MARL community. We further demonstrate how different distance measures capture the entanglement from different perspectives and give birth to the decomposition error in different norms.

\paragraph{\textbf{Applications of Markov Entanglement}}
We then apply Markov entanglement theory to several structured multi-agent MDPs. We prove a widely-used class of index policies is asymptotically separable, exhibiting a sublinear decomposition error scaling as $\gO(\sqrt N)$ with the number of agents $N$. This result theoretically justifies the practical effectiveness of value decomposition for index policies. Our proof builds on innovations that integrate Markov entanglement with mean-field analysis. 

For practitioners, we show that Markov entanglement can be efficiently estimated and serves as a surrogate to test whether value decomposition is feasible. Finally, we empirically demonstrate the low-entangled structure in several practical scenarios including a ride-hailing simulator.

\subsection{Other Related Work}
In the first section, we have reviewed typical empirical works on value decomposition. Here, we complement that discussion with related literature on theoretical insights. 

Prior theoretical research has extensively investigated the decomposition of optimal value functions in multi-agent settings. A prominent area involves Lagrange relaxation, with the Restless Multi-Armed Bandit (RMAB, \citealt{Whittle88restless}) as a foundational model. The per-agent decomposition error is proven to decay asymptotically to zero (\citealt{weber90on, Weber_Weiss_1991, Verloop2016AsymptoticallyOP}), justifying the asymptotic optimality of the well-known Whittle Index policy (\citealt{Whittle88restless}). The decay rate is further refined to quadratic or exponential under various conditions (\citealt{gast23exponential, gast24linear, brown2022dynamic, zhang2021restless, zhang2022near}). Other work generalizes to Weakly-Coupled MDPs (WCMDPs), deriving guarantees based on system structure (\citealt{santiago21, david25fluid, gast2022reoptimization}). However, \cite{daniel08relaxation} showed Lagrange relaxation can have arbitrarily large errors and proposed an alternative decomposition called Approximate Linear Programs (ALP). ALP is proven to have tighter error, a finding further explored by \cite{david23on}. Despite these advancements, characterizing decomposition error for general multi-agent MDPs remains challenging for both approaches. In contrast, our Markov entanglement theory analyzes value decomposition for general multi-agent MDPs under arbitrary policies, including optimal ones. Notably, we show sublinear decomposition error not only for the optimal Whittle Index policy but for any index policy as well.

Another line of theoretical work has concentrated on policy optimization via value decomposition. Despite the reported empirical successes, rigorous theoretical analysis remains challenging. \cite{baek2023policy} derived an approximation ratio for a specific index policy on a two-state RMAB. \cite{wang2021towards, dou2022understanding} analyzed the convergence of the CTDE paradigm under strong exploration assumptions, while also highlighting scenarios of divergence. In contrast, our work instead focuses on policy evaluation rather than optimization. This enables us to derive clear and interpretable bounds on the decomposition error for general finite-state multi-agent MDPs that only require the existence of a stationary distribution.

Finally, we note that value decomposition can be viewed as single-agent policy evaluation with linear one-hot feature approximation. While extensive research has analyzed the convergence of single-agent policy evaluation with linear functions (\citealt{john96analysis, jalaj21afinite, bertsekas1996neuro, srikant2019finite, liu2021temporal}), these results don't explain how the limit point effectively approximates the global value---which is the central focus of our work. Thus, our paper addresses an aspect orthogonal to this literature.

\subsection{Notations}
We abbreviate subscripts $(\vs)\coloneqq(s_{1:N})\coloneqq(s_1,s_2,\ldots,s_N)$. Particularly, for two-agent case, when the context is clear, we abbreviate $(\vs)\coloneqq(s_{AB})\coloneqq(s_A,s_B)$. Let $[N]=\{1,2,\ldots,N\}$ and $\mathbb Z^+$ be the set of positive integers. For a vector or matrix $\vx$, $\vx\geq0$ is equivalent to every element of $\vx$ to be non-negative and $\abs{\vx}$ denotes the element-wise absolute value. For (semi-)norm $\|\cdot\|_\alpha$ and norm $\|\cdot\|_\beta$, we define the $\alpha,\beta$-norm for matrix $\mA$ as $\|\mA\|_{\alpha,\beta}=\sup_{\|\vx\|_\beta=1}\|\mA\vx\|_\alpha$. We further abbreviate $\|\mA\|_\alpha\coloneqq\|\mA\|_{\alpha,\alpha}$. 
\section{Model}\label{sec: model}
We consider a standard two-agent MDP $\gM_{AB}(\gS,\gA,\mP,\vr_A,\vr_B,\gamma)$ with joint state space $\gS=\gS_A\times \gS_B$ and joint action space $\gA=\gA_A\times\gA_B$ where $A,B$ represent two agents. For simplicity, let $|\gS_A|=|\gS_B|=|S|$ and $|\gA_A|=|\gA_B|=|A|$. For agents at global state $\vs=(s_A,s_B)$ with action $\va=(a_A,a_B)$ taken, the system will transit to $\vs^\prime=(s_A^\prime,s_B^\prime)$ according to transition kernel $\vs^\prime\sim\mP(\cdot\mid \vs,\va)$ and each agent $i\in\{A,B\}$ will receive its local reward $r_i(s_i,a_i)$. The global reward $r_{AB}$ is defined as the summation of local rewards $r_{AB}(\vs,\va)\coloneqq r_A(s_A,a_A)+r_B(s_B,a_B)$, or in vector form, 
\[\vr_{AB}\in\R^{|S|^2|A|^2}\coloneqq \vr_A\otimes \ve + \ve\otimes \vr_B\,,\] where $\otimes$ is the tensor product and $\ve=\mathbf{1}\in\R^{|S||A|}$ is the vector of all ones. This reward structure is broadly satisfied in various piratical scenarios. For example, a ride-hailing platform's overall revenue is the summation of revenue of each driver. Many well-established models also include this reward structure, such as RMAB (\citealt{Whittle88restless,weber90on}) and WCMDPs (\citealt{brown2022dynamic,daniel08relaxation}). 
In Appendix~\ref{sec: shared reward}, we extend our results to multi-agent MDP model where the global cannot be decomposed. 
We further assume the local rewards are bounded, i.e. for agent $i\in \{A,B\}$, $\abs{r_{i}(s_i,a_i)}\leq r_{\max}^i$ for all $(s_i,a_i)$.

Given any global policy $\pi\colon \gS\to\Delta(\gA)$, the global Q-value under policy $\pi$ is defined as the discounted summation of global rewards,
\[Q^\pi_{AB}(\vs,\va)=\E\br{\sum_{t=0}^\infty \gamma^t r_{AB}(\vs^t,\va^t)\mid \pi,(\vs^0,\va^0)=(\vs,\va)}\,,\]
where $\gamma\in[0,1)$ is the discount factor. The value function is then defined as $V^\pi_{AB}(\vs)=\E_{\va\sim\pi(\cdot\mid\vs)}\br{Q^\pi_{AB}(\vs,\va)}$. We denote $\mP^\pi_{AB}\in\R^{|S|^2|A|^2\times |S|^2|A|^2}$ as the transition matrix induced by $\pi$ as $P^\pi_{AB}\pa{\vs^\prime,\va^\prime\mid \vs,\va}= \mP\pa{\vs^\prime\mid \vs,\va}\cdot\pi\pa{\va^\prime\mid \vs^\prime}$. Then by the Bellman Equation, we have $Q^\pi_{AB}(\vs,\va)=r_{AB}(\vs,\va)+\gamma \sum_{\vs^\prime,\va^\prime} P^\pi_{AB}\pa{\vs^\prime,\va^\prime\mid \vs,\va} Q^\pi_{AB}(\vs^\prime,\va^\prime)$, or in matrix form,
\begin{equation*}
    Q^\pi_{AB}=\pa{\mI-\gamma \mP^\pi_{AB}}^{-1} \vr_{AB}\,.
\end{equation*}
Our objective is to decompose this global Q-value $Q^\pi_{AB}$ as the summation of some local functions $Q_A$ and $Q_B$, i.e. $Q^\pi_{AB}(\vs,\va)= Q_A(s_A,a_A)+Q_B(s_B,a_B)$, or in vector form,
\begin{equation}\label{eq: Q-decomp}
    Q^\pi_{AB} = Q_A \otimes \ve + \ve\otimes Q_B\,.
\end{equation}
Notice we formally introduce our research question using Q-value instead of V-value function as in the introduction. Q-value decomposition is a stronger result that implies V-value function decomposition. It also turns out that Q-value further incorporates action information enabling more general theoretical analysis. More discussions can be found in Appendix~\ref{app: decomp_value}.

\subsection{Local (Q-)value Functions}
Recent literature offers several algorithms for learning local (Q-)values. In this paper, we use a meta-algorithm framework in~\ref{alg: meta} to summarize their underlying principles.

\begin{algorithm}[h]
\caption{Leaning Local Q-value Functions}
\begin{algorithmic}[1]\label{alg: meta}
\REQUIRE{Global policy $\pi$; horizon length $T$.}
\STATE Execute $\pi$ for $T$ epochs and obtain $\gD=\brk{(s_{AB}^t,a_{AB}^t,r_{AB}^t, s_{AB}^{t+1}, a_{AB}^{t+1})}_{t=1}^{T-1}$.
\STATE Each agent $i\in\brk{A,B}$ fits $Q^\pi_i$ using local observations $\gD_i=\brk{(s_i^t,a_i^t,r_i^t,s_{i}^{t+1}, a_i^{t+1})}_{t=1}^{T-1}$.
\end{algorithmic}
\end{algorithm}
This meta-algorithm framework is simple and intuitive: each agent independently fits its local Q-values based on local observations. Notably, the framework requires no prior knowledge of the MDP, and learning can be performed in a fully decentralized manner. Furthermore, we use term \emph{meta} in that we do not pose restrictions on how agents estimate their local Q-values. For tabular settings, one can plug in Temporal Difference (TD) learning (\citealt{Sutton18RL}) or its variants. For large-scale problems, one can apply linear function approximations (e.g. \citealt{baek2023policy, han22real, bertsekas1996neuro}) or more sophisticated neural networks (e.g. \citealt{qin20ride, peter18value, mahajan2019maven}).

Despite the flexibility in fitting local value functions, it is helpful to call out a particular approach: TD learning for local Q-values in the tabular case, as it facilitates the analysis and reveals the structure of value decomposition in the next section.

\paragraph{Local TD learning.} Although each agent's environment is not Markovian in a local sense (it is, more precisely, partially observed Markovian), TD learning views local observations $\gD_i$ as being sampled from a Markov chain. As a result, an agent's local transition is a marginalization of the global transition. We focus on this ``marginalized'' local transition matrix under the stationary distribution. Mathematically, for agent $A$, we denote $\mP^\pi_{A}\in\R^{|S||A|\times|S||A|}$ as its local transition where
\begin{equation}\label{eq: PA}
    P^\pi_A(s_{A}^\prime,a_A^\prime\mid s_A,a_A) = \sum_{s_B^\prime,a_B^\prime} \sum_{s_B,a_B}P^\pi_{AB}\pa{s_{AB}^\prime,a_{AB}^\prime\mid s_{AB},a_{AB}} \mu^\pi_{AB}(s_B,a_B\mid s_A,a_A)\,.
\end{equation}
Here, $\mu_{AB}^{\pi} \in \Delta(\mathcal{S})$ denotes the global stationary distribution under policy $\pi$ (for convenience, we assume $\pi$ induces a unichain, i.e. $\muAB$ is unique and strictly positive). We further define the occupancy measure under policy $\pi$ as $\mu^\pi_{AB}(\vs,\va)=\mu^\pi_{AB}(\vs)\pi(\va\mid \vs)$. Thus~\eqref{eq: PA} states the local transition of agent $A$ at pair $(s_A, a_A)$ is a projection of global transition weighted by the conditional occupancy measure of agent~$B$'s state-action pairs, $\mu^\pi_{AB}(s_B, a_B \mid s_A, a_A)$.\footnote{For $\mu^\pi_{AB}(s_B,a_B\mid s_A,a_A)$ to be well-defined, we require $\mu^\pi_{AB}(s_A,a_A)>0$. If $\mu^\pi_{AB}(s_A, a_A) = 0$, the action $a_A$ is never taken in state $s_A$ under policy $\pi$, and we exclude such pairs by restricting the feasible action set $\mathcal{A}(s_A)$. All theoretical results hold for the remaining valid state-action pairs.   } Given this ``marginalized" local transition, the local Q-values obtained by Meta Algorithm~\ref{alg: meta} using tabular TD learning converge to the solution of the following ``marginalized'' Bellman equation:
\begin{equation*}
    Q^\pi_{A}=\pa{\mI-\gamma \mP^\pi_{A}}^{-1} \vr_{A}\,.
\end{equation*}
By symmetry, we can derive analogous results for agent~$B$, obtaining its local transition matrix $\mP^\pi_B$ and local Q-values $Q^\pi_B$. Next, we show how $Q^{\pi}_A$ and $Q^{\pi}_{B}$ contribute to the exact value decomposition.

\section{Exact Value Decomposition}
In this section, we investigate whether the global Q-value admits an exact decomposition, i.e. $Q^\pi_{AB}=Q_A\otimes \ve +\ve\otimes Q_B$ for some local functions $Q_A$ and $Q_B$. As demonstrated in the introduction, we identify a key condition called \emph{Markov Entanglement}, which we formally define as follows:

\begin{definition}[Two-agent Markov Entanglement]\label{def: two-agent ME}
    Consider a two-agent MDP $\gM_{AB}$ and policy $\pi\colon \gS \to \Delta(\gA)$, the two agents are \textbf{separable} if there exists $K\in\mathbb Z ^+$, measure $\{x_j\}_{j\in[K]}$ satisfying $\sum_{j=1}^Kx_j=1$, and transition matrices $\left\{\mP^{(j)}_A,\mP^{(j)}_B\right\}_{j\in[K]}$ such that 
    \[\PAB = \sum_{j=1}^K x_j\mP_A^{(j)}\otimes \mP_B^{(j)}\,.\]
    If there exists no such decomposition, the two agents are  \textbf{entangled}.
\end{definition}


Our first theorem shows that an MDP with no Markov entanglement is indeed sufficient for the exact value decomposition.
\begin{theorem}\label{thm: mixed_state}
    Consider a two-agent MDP $\gM_{AB}$ and policy $\pi$. If two agents are separable, i.e. there exists $K\in \mathbb Z^+$, measure $\{x_j\}_{j\in[K]}$, and transition matrices $\left\{\mP^{(j)}_A,\mP^{(j)}_B\right\}_{j\in[K]}$ such that 
    $\PAB = \sum_{j=1}^K x_j\mP_A^{(j)}\otimes \mP_B^{(j)}$. Then it holds 
    \[\PA = \sum_{i=1}^K x_j\mP_A^{(j)}\,,\qquad \PB = \sum_{j=1}^K x_j\mP_B^{(j)}\,.\]
    Furthermore, the \Eqref{eq: Q-decomp} holds\[Q_{AB}^\pi = Q^\pi_{A}\otimes \ve + \ve\otimes Q^\pi_{B}\,.\]
\end{theorem}

This theorem establishes that even when the system is not independent, as long as it can be represented as a \emph{linear combination of independent subsystems}, the global Q-value admits an exact decomposition. More importantly, this result implies that local TD learning (or Meta Algorithm~\ref{alg: meta} more generally) will converge to the desired local transition matrices. Consequently, if an exact decomposition of $Q^\pi_{AB}$ exists, Meta Algorithm~\ref{alg: meta} is guaranteed to recover the corresponding local Q-values $Q_A^\pi$ and $Q_B^\pi$.

\subsection{An Illustrative Example of Coupling and Markov Entanglement}
To elucidate the concept of Markov entanglement, we present an example of two-agent MDP where agents are coupled but not entangled. 

Consider a two-agent MDP $\gM_{AB}$ with $\|\gA_A\|=\|\gA_B\|=2$ , where action $1$ means activate and $0$ means idle. Each agent $i\in\{A,B\}$ has its own transition kernel $\mP_i$. We examine the following policy:

\begin{example}[Shared Randomness]\label{example: shared random}
    At each time-step, randomly activate one agent and keep another idle. In other words,
\begin{equation*}
\pi(\va\mid\vs)=\left\{\begin{matrix}
 1/2  & \qquad\va=(0,1) \textrm{ or } \va=(1,0)\,,\\
 0 & \textrm{otherwise.}
\end{matrix}\right.
\end{equation*}
\end{example}
As a concrete example, consider a ride-hailing platform where the policy randomly assigns one of two available drivers to each incoming order. Formally, this policy couples the agents through the constraint $a_A + a_B = 1$ at each timestep. However, we will demonstrate that despite this coupling, there's \emph{no} entanglement.
Specifically, we construct the following decomposition
\begin{equation}\label{eq: random 1/2}
    \mP^\pi_{AB}=\frac{1}{2} \mP^0_A\otimes \mP^1_B + \frac{1}{2}\mP^1_A\otimes\mP^0_B\,,
\end{equation}
where $\mP^0_A$ refers to the transition matrix of agents $A$ taking action $a=0$ and we similarly define $\{\mP^a_i\}_{i\in\{A,B\}, a\in \{0,1\}}$. Intuitively, the right-hand side of \eqref{eq: random 1/2} describes how at each time step, the global system randomly selects between two possible transitions: $\mP^0_A\otimes \mP^1_B$ or $\mP^1_A\otimes \mP^0_B$, each with equal probability (akin to rolling a fair dice). This interpretation aligns with the random policy introduced in Example~\ref{example: shared random} and one can also formally show
\begin{align*}
    & \pa{\frac{1}{2} \mP^0_A\otimes \mP^1_B + \frac{1}{2}\mP^1_A\otimes\mP^0_B}{(\vs^\prime,\va^\prime\mid \vs,\va)}\\
    =& P_A(s_A^\prime\mid s_A,a_A) P_B(s_B^\prime\mid s_B,a_B) \pa{\frac{1}{2}\pi_0(a_A^\prime\mid s_A^\prime)\pi_1(a_B^\prime\mid s_B^\prime) + \frac{1}{2} \pi_1(a_A^\prime\mid s_A^\prime)\pi_0(a_B^\prime\mid s_B^\prime)}\\
    =& P_A(s_A^\prime\mid s_A,a_A) P_B(s_B^\prime\mid s_B,a_B) \pi(\va^\prime\mid\vs^\prime)=\mP^\pi_{AB}(\vs^\prime,\va^\prime\mid \vs,\va)\,,
\end{align*}
where $\pi_0(a\mid s)=\mathbf{1}\brk{a=0}$ and $\pi_1(a\mid s)=\mathbf{1}\brk{a=1}$. This example thus clearly demonstrates a \emph{coupled} system can still be \emph{separable}. As a result, exact Q-value decomposition holds and applying Meta Algorithm~\ref{alg: meta} will give us local Q-values for unbiased policy evaluation.

Finally, this policy is named \emph{Shared Randomness}, a concept that also has a direct parallel in quantum physics. Shared randomness plays a crucial role in characterizing the fundamental distinction between classical correlation and quantum entanglement in the quantum world, e.g. the famous Einstein-Podolsky-Rosen (EPR) paradox. When it comes to the Markov world, it delineates the difference between coupling and Markov entanglement.

\subsection{Proof of Sufficiency}

Theorem~\ref{thm: mixed_state} admits a simple proof based on the several basic properties of tensor product. First of all, given $\PAB = \sum_{j=1}^K x_j\mP_A^{(j)}\otimes \mP_B^{(j)}$, we have 
    \begin{align*}
P^\pi_{AB}\pa{s_A^\prime,s_B^\prime,a_A^\prime,a_B^\prime\mid s_A,s_B,a_A,a_B} & = \sum_{j=1}^K x_jP_A^{(j)}(s_A^\prime,a_A^\prime\mid s_A,a_A) P_B^{(j)}(s_B^\prime,s_B^\prime\mid s_B,a_B)\,.
    \end{align*}
    Recall $\mP_A^\pi$ in \Eqref{eq: PA}, it's evident that
    \begin{align*}
        P^\pi_A(s_{A}^\prime,a_A^\prime\mid s_A,a_A)
        &= \sum_{s_B^\prime,a_B^\prime} \sum_{s_B,a_B}\sum_{j=1}^K x_jP_A^{(j)}(s_A^\prime,a_A^\prime\mid s_A,a_A) P_B^{(j)}(s_B^\prime,s_B^\prime\mid s_B,a_B) \mu^\pi_{AB}(s_B,a_B\mid s_A,a_A)\\
        &= \sum_{j=1}^K x_jP_A^{(j)}(s_A^\prime,a_A^\prime\mid s_A,a_A) \sum_{s_B,a_B} \mu^\pi_{AB}(s_B,a_B\mid s_A,a_A) \sum_{s_B^\prime,a_B^\prime}P_B^{(j)}(s_B^\prime,s_B^\prime\mid s_B,a_B)\\
        &= \sum_{j=1}^K x_iP_A^{(j)}(s_A^\prime,a_A^\prime\mid s_A,a_A)\,,
    \end{align*}
    where the second last equation holds by rearranging the summation. This leads to $\PA = \sum_{i=1}^K x_i\mP_A^{(i)}$. It remains to show \Eqref{eq: Q-decomp}, and notice that 
    \begin{align*}
        \pa{\mI-\gamma\PAB}^{-1} \pa{\vr_A\otimes \ve} 
        &= \sum_{t=0}^\infty \gamma^t\pa{\sum_{j=1}^K x_j\mP_A^{(j)}\otimes \mP_B^{(j)}}^t \pa{\vr_A\otimes \ve}\\
        &\overset{(i)}{=} \sum_{t=0}^\infty \gamma^t\pa{\pa{\sum_{j=1}^K x_j\mP_A^{(j)}}^t \vr_A}\otimes \ve\\
        &= \pa{\pa{\mI-\gamma\PA}^{-1} \vr_A}\otimes \ve= Q^\pi_{A}\otimes \ve\,,
    \end{align*}
    where we refer to $(i)$ as an ``absorbing" technique based on the bilinearity and mixed-product property of tensor product\footnote{We introduce several basic properties of tensor product in Appendix~\ref{app: linaer algebra}. }. Specifically, since $\mP \ve = \ve$ for any transition matrix $\mP$, we have for any $t$, 
    \begin{align*}
        &\pa{\sum_{j=1}^K x_j\mP_A^{(j)}\otimes \mP_B^{(j)}}^t \pa{\vr_A\otimes \ve} \\
        =& \pa{\sum_{j=1}^K x_j\mP_A^{(j)}\otimes \mP_B^{(j)}}^{t-1}\pa{\sum_{j=1}^K x_j \pa{ \mP_A^{(j)}\vr_A}\otimes\pa{\mP_B^{(j)}\ve}}\\
        =& \pa{\sum_{j=1}^K x_j\mP_A^{(j)}\otimes \mP_B^{(j)}}^{t-1}\pa{\sum_{j=1}^K x_j  \mP_A^{(j)}\vr_A}\otimes\ve\\
        =&\ldots = \pa{\pa{\sum_{j=1}^K x_j\mP_A^{(j)}}^t \vr_A}\otimes \ve\,.
    \end{align*}
    Similar results can be derived for $\PB$ such that $\pa{\mI-\gamma\PAB}^{-1} \pa{\ve\otimes\vr_B} = \ve\otimes Q^\pi_{B} $.  Finally, combining the above results, we have 
    \begin{align*}
        Q_{AB}^\pi  = \pa{\mI-\gamma\PAB}^{-1} \vr_{AB} =  \pa{\mI-\gamma\PAB}^{-1}\pa{\vr_A\otimes \ve+\ve \otimes \vr_B} = Q^\pi_{A}\otimes \ve + \ve\otimes Q^\pi_{B}\,.
    \end{align*}

\subsection{Comparisons with Quantum Entanglement}\label{sec: neg_coeff}
It turns out that our Markov entanglement condition serves as a mathematical counterpart of quantum entanglement in quantum physics. We provide the formal definition of the latter for comparison..  
\begin{definition}[Two-party Quantum Entanglement]
 Consider a two-party quantum system composed of two subsystems $A$ and $B$. The joint state $\rho_{A B}$ is \textbf{separable} if there exists $K\in\mathbb Z^+$, a probability measure $\left\{x_j\right\}_{j\in[K]}$, and density matrices $\left\{\rho^{(j)}_A,\rho^{(j)}_B\right\}_{j\in[K]}$ such that
\[
\rho_{AB}=\sum_{j=1}^K x_j \rho^{(j)}_A \otimes \rho^{(j)}_B\,.\]
If there exists no such decomposition, $\rho_{AB}$ is \textbf{entangled}.
\end{definition}
The density matrices are square matrices satisfying certain properties such as positive semi-definiteness and trace normalization, which can be viewed as the counterparts of transition matrices in the Markov world. Despite the similarities in mathematical form, quantum entanglement imposes an additional constraint requiring $\left\{x_j\right\}_{j\in[K]}$ to be a probability measure, i.e. $\vx\geq 0$. In contrast, our Markov entanglement defined in Definition~\ref{def: two-agent ME} permits general linear coefficients $\left\{x_j\right\}_{j\in[K]}$ as long as $\sum_{j=1}^kx_j=1$. This distinction raises the important question of whether negative coefficients are indeed necessary in characterizing Markov entanglement.

To start with, we introduce the set of all separable transition matrices
\[\gP_{\textrm{SEP}}=\left\{ \mP\geq 0 \;\middle|\; \mP=\sum_{j=1}^Kx_j\mP_A^{(j)}\otimes \mP_B^{(j)} \;,\; \sum_{j=1}^Kx_j=1  \right\}\,,\]
where $K\in\mathbb Z^+$ and $\left\{\mP^{(j)}_A,\mP^{(j)}_B\right\}_{j\in[K]}$ are transition matrices. $\mP\geq0$ calls for every element of $\SEP$ to be a valid transition matrix. It's clear that a transition matrix $\mP^\pi_{AB}$ is separable if and only if $\mP^\pi_{AB}\in\gP_{\textrm{SEP}}$. On the other hand, a direct analogy of quantum entanglement gives us the following set that further requires non-negative coefficients,
\[\gP_{\textrm{SEP}}^+=\left\{ \mP\geq 0 \;\middle|\; \mP=\sum_{j=1}^Kx_j\mP_A^{(j)}\otimes \mP_B^{(j)} \;,\; \sum_{j=1}^Kx_j=1 \;,\;\vx\geq0 \right\}\,.\]
Interestingly, it turns out $\gP_{\textrm{SEP}}^+ \subsetneq\gP_{\textrm{SEP}}$. In other words, there exist separable two-agent MDPs that can only be represented by linear combinations but not convex combinations of independent subsystems, as we demonstrate in Appendix~\ref{app: neg_coeff}. This result justifies the necessity of negative coefficients in $\vx$ and highlights a structural difference between Markov entanglement and quantum entanglement.

\section{Necessary Condition for the Exact Value Decomposition}
In this section, we investigate whether Markov entanglement is necessary for the exact Q-value decomposition. The answer is in general no, since one can construct trivial counterexamples such as $\vr_A=\vr_B=\vzero$ or $\gamma =0$, where the decomposition trivially holds. These trivial examples highlight the impact of specific reward or $\gamma$ on the value decomposition. On the other hand, we focus on a stronger and more general concept of the exact value decomposition that holds under any reward kernel given $\gamma>0$. Formally, we present the following theorem. 
\begin{theorem}\label{thm: necessary_condition}
    Consider a two-agent Markov MDP $\gM_{AB}$ with discount factor $\gamma >0$ and $\pi\colon \gS \to \Delta(\gA)$. Suppose there exists local functions $Q_i\colon \vr_i\to\R^{|S||A|}$ for $i\in \{A,B\}$ such that $Q_{AB}^\pi = Q_A(\vr_A)\otimes \ve + \ve\otimes Q_{B}(\vr_B)$ holds for any pair of reward $\vr_A,\vr_B$,
    then $A,B$ must be separable. 
\end{theorem}

Combined with Theorem~\ref{thm: mixed_state}, we conclude Markov entanglement serves as a sufficient and necessary condition for the exact value decomposition. We also emphasize that Theorem~\ref{thm: necessary_condition} considers general local functions $Q_i$. This generality accommodates all methods for fitting local $Q_i$, such as deep neural networks, provided that the training relies solely on the local observations of agent $i$. Nevertheless, Theorem~\ref{thm: mixed_state} guarantees that if $Q^\pi_{AB}$ admits the exact value decomposition under arbitrary reward, we can obtain corresponding local Q-values through Meta Algorithm~\ref{alg: meta} using local TD Learning.


There exist other possible ways for value decomposition. For example, \cite{peter18value, dou2022understanding} consider $Q^\pi_{AB}(\vs,\va)=L_A(s_A,a_A,\vr_{AB})+L_B(s_B,a_B,\vr_{AB})$ where $L_A,L_B$ are learned jointly via minimizing the global Bellman error\footnote{In Appendix~\ref{app: decomp_general}, we provide an example of entangled MDP that allows for such decomposition where $L_A$ depends on both $\vr_A$ and $\vr_B$. }; \cite{rashid2020monotonic, mahajan2019maven, son19qtran, wang2020qplex} consider general monotonic operations beyond additive decompositions. These methods introduce possibly richer representations at the cost of more sophisticated implementations and less interpretability, which is beyond the scope of this paper.

\subsection{Proof Sketch of Necessity}

Our proof of necessity builds on several novel techniques. We provide an overview here and the full proof is delayed to Appendix~\ref{app: proof_thm_2}.

\paragraph{Step 1: Understanding the Orthogonal Complement.} If a transition matrix is entangled, it will have non-zero component in the orthogonal complement of $\SEP$, which we construct as
\[\SEP^\perp=\left\{\sum_{j=1}^{|S||A|-1}\pa{\varepsilon_j \ve^\top}\otimes \mW^1_j + \sum_{j=1}^{|S||A|-1}\mW^2_j\otimes\pa{\varepsilon_j \ve^\top} \;\middle|\; W^1_{1:j},W^2_{1:j}\in\R^{|S||A|\times |S||A|}\right\}\,,\]
where $\varepsilon_j=(1,0,\ldots,0,-1,0,\ldots,0)^\top$ with the first element $1$ and $(j+1)$-th element $-1$. 

Then, we study an intermediate transition matrix $(1-\gamma) (\mI-\gamma\PAB)^{-1}$. If it's entangled, there will be non-zero component $\mY\in \SEP^\perp \ne 0$ such that $\Tr(\mY^\top(\mI-\gamma \mP^\pi_{AB})^{-1})\ne 0$. We then apply singular value decomposition to $\mY$. It follows there exists some $j$ and $\vu,\vv\in\R^{|S||A|}$ such that either $\Tr(\pa{ \ve\varepsilon_j^\top}\otimes \pa{\vv\vu^\top} (\mI-\gamma \mP^\pi_{AB})^{-1})\ne0$ or $\Tr( \pa{\vv\vu^\top}\otimes \pa{\ve\varepsilon_j^\top}(\mI-\gamma \mP^\pi_{AB})^{-1})\ne0$. In either case, we can construct $\vr_A,\vr_B$ based on $\vu,\vv$ such that $Q^\pi_{AB}$ is not decomposable under this pair of rewards. This forms a contradiction and we conclude 
\[\textrm{decomposable }Q^\pi_{AB}\Longrightarrow\textrm{ separable }(1-\gamma) (\mI-\gamma\PAB)^{-1}\,.\]

\paragraph{Step 2: Connecting to ``Inverse".}
Finally, we complete the proof via connecting $(1-\gamma) (\mI-\gamma \PAB)^{-1}$ with $\PAB$. Specifically, we prove the following lemma
\[
\textrm{separable }(1-\gamma) (\mI-\gamma\PAB)^{-1}\Longleftrightarrow\textrm{ separable } \PAB\,.\]
The $\Longleftarrow$ side is straightforward since $(\mI-\gamma\PAB)^{-1}$ is the Neumann series of $\gamma\PAB$. For the converse $\Longrightarrow$, we seek to invert this Neumann series. This is achieved by a careful analysis of the operator norm of $\mI-(1-\gamma) (\mI-\gamma\PAB)^{-1}$.

%

\section{Value Decomposition Error in General Two-agent MDPs}


In general, the system transition $\PAB$ can be arbitrarily complex, such that the agents are entangled with each other, i.e. there exists no such decomposition of the form  $\mP^\pi_{AB}=\sum_{j=1}^Kx_j\mP^{(j)}_A\otimes\mP^{(j)}_B$. In these scenarios, we investigate when the value decomposition $Q^\pi_A\otimes \ve + \ve \otimes Q^\pi_B$ is an effective approximation of $Q^\pi_{AB}$. Intuitively, more separable systems should exhibit smaller decomposition errors. When the system is completely separable, we recover exact value decomposition as in Theorem~\ref{thm: mixed_state}. To formalize and quantify this intuition, we proposed the measure of Markov entanglement for general two-agent MDPs in the introduction,
\begin{definition}[Measure of Two-agent Markov Entanglement]
Consider a two-agent MDP $\gM_{AB}$ and policy $\pi$. Its measure of Markov entanglement is defined as
\begin{equation}\label{eq: two agent measure}
    E(\PAB)\coloneqq\min_{\mP\in\gP_{\rm{SEP}}} d(\PAB, \mP)\,,
\end{equation}
where $\SEP$ is the set of all separable transition matrices and $d(\cdot,\cdot)$ is some distance measure.
\end{definition}
This concept can also find its counterpart in quantum physics,
\[E(\rho_{AB})\coloneqq \min_{\rho\in\rho_{\rm{SEP}}} d(\rho_{AB}, \rho)\,,\]
where $\rho_{\rm{SEP}}$ is the set of all separable quantum states. In quantum physics, various distance measures have been developed for density matrices, including quantum relative entropy, Bures distance, and trace distance (\citealt{nielsen2010quantum}). These different distance measures embody distinct interpretations of the physical world. 

Following this conceptual framework, we examine several distance measures for transition matrices in the subsequent subsections. Ultimately, we will demonstrate a series of theorems that fall into the following framework
\[\Big\lVert Q^\pi_{AB}-\pa{Q^\pi_A\otimes \ve + \ve \otimes Q^\pi_B} \Big\rVert=\gO\Big(E(\PAB)\Big)\,.\]
This result reduces the problem of bounding decomposition error to the analysis of Markov entanglement, providing theorists with a general framework to characterize value decomposition error in two-agent MDPs.

\subsection{Total Variation Distance}
We begin by considering the Total Variation (TV) distance, a widely used metric for transition matrices. It is defined as the maximum total variation distance between any pair of corresponding row transition probability distributions.
\begin{definition}[Total Variation Distance between Transition Matrices]\label{def: TVD}
    The total variation distance between two transition matrices $\mP,\mP^\prime\in\R^{|S||A|\times |S||A|}$ is defined as
    \[\norm{\mP-\mP^\prime}_{\rm{TV}} \coloneqq \max_{(s,a)\in\gS\times\gA} D_{\rm{TV} }(\mP(\cdot,\cdot\mid s,a), \mP^\prime(\cdot,\cdot\mid s,a))\,,\]where $D_{\rm{TV}}$ is the total variation distance between probability measures.
\end{definition}
Equipped with this TV distance measure, we can plug it into \eqref{eq: two agent measure} and obtain the measure of entanglement w.r.t TV distance, given by $E(\PAB)=\min_{\mP\in\gP_{\rm{SEP}}} \norm{\PAB- \mP}_{\textrm{TV}}$. The following theorem connects this measure to the value decomposition error. 
\begin{theorem}\label{thm: two-agent value decomp}
    Consider a two-agent MDP $\gM_{AB}$ and policy $\pi\colon \gS \to \Delta(\gA)$ with the measure of Markov entanglement $E(\PAB)$ w.r.t the total variation distance, it holds
    \[\norm{\mP^\pi_A-\sum_{j=1}^Kx_j\mP^{(j)}_A }_{\rm{TV}}\leq E(\PAB)\,,\qquad \norm{\mP^\pi_B-\sum_{j=1}^Kx_j\mP^{(j)}_B }_{\rm{TV}}\leq E(\PAB)\,,\]
    where $\sum_{j=1}^Kx_j\mP^{(j)}_A\otimes \mP^{(j)}_B$ is an optimal solution of \eqref{eq: two agent measure} under total variation distance. Furthermore, the decomposition error is entry-wise bounded by the measure of Markov entanglement,
    \[\Big\lVert Q^\pi_{AB}-\pa{Q^\pi_A\otimes \ve + \ve \otimes Q^\pi_B} \Big\rVert_\infty \leq \frac{4\gamma E(\PAB)(r_{\max}^A+r_{\max}^B)}{(1-\gamma)^2}\,.\]
\end{theorem}
This theorem offers a sharp characterization of the relationship between Q-value decomposition error and the measure of Markov entanglement w.r.t TV distance. First, Theorem~\ref{thm: two-agent value decomp} extends Theorem~\ref{thm: mixed_state}. Notice that $E(\PAB)=0$ is equivalent to $\PAB$ being separable. We thus recover the condition of Theorem~\ref{thm: mixed_state} and obtain the exact decomposition as shown in \eqref{eq: Q-decomp}. Furthermore, when the system is entangled, we show that the local transitions $\PA,\PB$ learned by Meta Algorithm~\ref{alg: meta} introduce a bias that can be bounded by $E(\PAB)$ in terms of TV distance. Finally, we derive an entry-wise (i.e. $\ell_\infty$-norm) bound on the value decomposition error, ensuring that the error at each state-action pair is controlled by the TV distance up to a constant factor.


\subsection{Agent-wise Distance}
While the total variation distance in Definition~\ref{def: TVD} serves as a simple and intuitive measure for transition matrices, we demonstrate that a more refined distance measure can be established for multi-agent MDPs. To this end, we introduce the following distance measure.

\begin{definition}[Agent-wise Total Variation Distance]\label{def: ATVD}
    The agent-wise total variation distance between two transition matrices $\mP,\mP^\prime\in\R^{|S|^2|A|^2\times |S|^2|A|^2}$ w.r.t agent $A$ is defined as
    \begin{align*}
        \norm{\mP-\mP^\prime}_{\rm{ATV}_A} 
        \coloneqq \max_{(\vs,\va)\in\gS\times\gA}D_{\rm{TV}}\pa{\sum_{s_{B}^\prime,a_{B}^\prime}\mP(\cdot,\cdot\mid \vs,\va), \sum_{s_{B}^\prime,a_{B}^\prime}\mP^\prime(\cdot,\cdot\mid \vs,\va)}\,.
    \end{align*}
\end{definition}
The agent-wise total variation (ATV) distance w.r.t agent $B$ can be defined similarly. Intuitively, compared to TV, ATV focuses on individual agents and measures the difference between their local transitions. We can also plug agent-wise TV into \eqref{eq: two agent measure} and obtain the Markov entanglement measurement w.r.t ATV distance $E_A(\PAB)\coloneqq\min_{\mP\in\gP_{\rm{SEP}}} \norm{\PAB- \mP}_{\textrm{ATV}_A}$. In fact, one can verify
\begin{align}
    E_A(\PAB)&=\min_{\mP\in\gP_{\rm{SEP}}} \norm{\PAB- \sum_{j=1}^Kx_j\mP_A^{(j)}\otimes \mP_B^{(j)}}_{\textrm{ATV}_A}\notag\\
    &= \min_{\mP_A} \max_{(\vs,\va)\in\gS\times\gA}D_{\rm{TV}}\Big(\PAB(\cdot,\cdot\mid \vs,\va), \mP_A(\cdot,\cdot\mid s_A,a_A)\Big)\,,\label{eq: degree of independent}
\end{align}
where the last $D_{\rm{TV}}$ is taking over support $\gS_A\times\gA_A$. \eqref{eq: degree of independent} implies that the optimal solution for \eqref{eq: two agent measure} under ATV distance depends solely on the closest local transition $\mP_A$. Recall that if agent $A$ is independent of the system, its local transition only depends on its local state and action $(s_A,a_A)$ rather than $(\vs,\va)$. Thus, \eqref{eq: degree of independent} essentially quantifies \emph{how closely agent $A$ can be approximated as an independent subsystem}. 

Furthermore, it turns out ATV is a tighter distance compared to the original TV distance, i.e. $\norm{\mP-\mP^\prime}_{\rm{ATV}_A}\leq \norm{\mP-\mP^\prime}_{\rm{TV}}$. This comes from the fact that ATV takes the supremum over an aggregated subset of events compared to TV. As a result, the measure of Markov entanglement w.r.t ATV distance is also smaller than that w.r.t TV distance. 
This enables us to derive a stronger version of Theorem~\ref{thm: two-agent value decomp} using ATV distance.

\begin{theorem}\label{thm: two-agent atv}
    Consider a two-agent MDP $\gM_{AB}$ and policy $\pi\colon \gS \to \Delta(\gA)$ with the measure of Markov entanglement $E_A(\PAB),E_B(\PAB)$ w.r.t the agent-wise total variation distance, it holds
    \[\norm{\mP^\pi_A-\mP_A }_{\rm{TV}}\leq E_A(\PAB)\,,\qquad \norm{\mP^\pi_B-\mP_B }_{\rm{TV}}\leq E_B(\PAB)\,.\]
    where $\mP_A,\mP_B$ are the optimal solutions of \eqref{eq: two agent measure} under agent-wise total variation distance w.r.t agent $A,B$ respectively. Furthermore, the decomposition error is entry-wise bounded by the measure of Markov entanglement,
    \[\Big\lVert Q^\pi_{AB}-\pa{Q^\pi_A\otimes \ve + \ve \otimes Q^\pi_B} \Big\rVert_\infty \leq \frac{4\gamma \pa{E_A(\PAB)r_{\max}^A+E_B(\PAB)r_{\max}^B}}{(1-\gamma)^2}\,.\]
\end{theorem}

 The bound derived using ATV is more refined and explicitly takes into account the multi-agent structure. Each agent $i\in\{A,B\}$ contributes to the decomposition error by its entanglement with the system, i.e. $E_i(\PAB)$, weighted by its maximum local reward $r^i_{\max}$. Furthermore, given that ATV is a tighter distance compared to TV, i.e. $E_i(\PAB)\leq E(\PAB)$ for $i\in\{A,B\}$, Theorem~\ref{thm: two-agent atv} provides a tighter bound for the entry-wise decomposition error compared to Theorem~\ref{thm: two-agent value decomp}.

\subsection{Proof Sketch of Theorem~\ref{thm: two-agent atv}}
In this section, we briefly work through the main proof framework and key techniques. The full proof is delayed to Appendix~\ref{app: Proof of atv}. To begin, we can partition the decomposition error related to agent $A$ into two terms,
\begin{align*}
        &\pa{\mI-\gamma \mP^\pi_{AB}}^{-1} (\vr_A\otimes \ve) - \pa{\pa{\mI-\gamma \mP^\pi_A}^{-1}\vr_A} \otimes \ve\\
        = & \pa{\mI-\gamma \mP^\pi_{AB}}^{-1} (\vr_A\otimes \ve)- \pa{\mI-\gamma\mP_A\otimes\mP_B  }^{-1} (\vr_A\otimes \ve)  +\pa{\mI-\gamma\mP_A\otimes\mP_B  }^{-1} (\vr_A\otimes \ve) - \pa{\pa{\mI-\gamma \mP^\pi_A}^{-1}\vr_A} \otimes \ve\\
        \stackrel{(i)}{=} & \underbrace{(\mI-\gamma \mP^\pi_{AB})^{-1} (\vr_A\otimes \ve)- \pa{\mI-\gamma\mP_A\otimes\mP_B  }^{-1} (\vr_A\otimes \ve)}_{(I)}  +\underbrace{\pa{\pa{\mI-\gamma\mP_A  }^{-1} \vr_A}\otimes \ve - \pa{\pa{\mI-\gamma \mP^\pi_A}^{-1}\vr_A} \otimes \ve}_{(II)}\,,
    \end{align*}where $(i)$ follows the same ``absorbing'' technique in the proof of Theorem~\ref{thm: mixed_state}. It therefore suffices to bound $(I)$ and $(II)$ in $\|\cdot\|_\infty$-norm separately.

Notice both $(I)$ and $(II)$ call for bounding the difference between matrix inverses. We thus introduce the following perturbation lemma. 
\begin{lemma}[Lemma 1 in \cite{farias2023correcting}]\label{lem: matrix inverse}
    Let $\mP,\mP^\prime\in\R^{n\times n}$ such that $(\mI-\mP)^{-1}$ and $(\mI-\mP^\prime)^{-1}$ exist. Then it holds
    \[(\mI-\mP^\prime)^{-1}-(\mI-\mP)^{-1} = (\mI-\mP^\prime)^{-1}(\mP^\prime-\mP)(\mI-\mP)^{-1}\,.\]
\end{lemma}
This lemma converts the inverse error to more related perturbation term $\mP^\prime-\mP$. Our last technique connects the $\|\cdot\|_\infty$-norm bound of this perturbation term with TV/ATV distance. Specifically, we can rewrite the TV distance as a constraint optimization problem,
\[\norm{\mP-\mP^\prime}_{\rm{TV}}=\frac{1}{2}\norm{\mP-\mP^\prime}_{\infty}=\frac{1}{2}\sup_{  \|\vx\|_\infty=1  }\|(\mP-\mP^\prime)\vx\|_\infty\,.\]
where $\vx\in\R^{|S|^2|A|^2}$. The first equation comes from the relationship between TV distance and variation distance. The last equation is the definition of $\norm{\cdot}_\infty$. We can also rewrite ATV distance similarly,
\begin{equation}\label{eq: ATV-optimization}
\norm{\mP-\mP^\prime}_{\rm{ATV}_A} =\frac{1}{2} \sup_{\|\vx\|_\infty=1 }\|\pa{\mP-\mP^\prime}(\vx\otimes \ve)\|_\infty\,.
\end{equation}
Note that for ATV, the feasible zone $\vx\in\R^{|S||A|}$. Finally, putting it together, we are able to bound $(I)$ and $(II)$ as well as the decomposition error related to agent $B$, which concludes the proof.

\subsection{Distance Weighted by Stationary Distribution}
In the preceding subsections, we discussed the (agent-wise) total variation distance for transition matrices. These distance measures impose a requirement of a uniformly bounded total variation distance across all state-action pairs, which consequently leads to strong entry-wise error bounds for Q-value decomposition. However, this uniform TV distance bound can sometimes be overly restrictive for general two-agent MDPs.

As an alternative, we aim to trade a weaker error bound for the value decomposition for a less stringent condition on the system transition. To achieve this, a practical choice is to consider an error weighted by the stationary distribution. Formally, we introduce the following norm.
\begin{definition}[$\mu$-norm]
    Given a transition matrix $\mP\in\R^{|\gS||\gA|\times|\gS||\gA|}$ with occupancy measure\footnote{Since $\mu\in\R^{|\gS||\gA|}$ is the stationary distribution of $\mP\in\R^{|\gS||\gA|\times|\gS||\gA|}$, we use ``stationary distribution" and ``occupancy measure" exchangeably when the context is clear.} $\mu\in\R^{|\gS||\gA|}$, for any vector $\vx\in\R^{|\gS||\gA|}$ the $\mu$-norm is defined as
    \begin{equation*}
        \|\vx\|_\mu\coloneqq \sum_{(s,a)\in\gS\times\gA} \mu(s,a)\abs{x(s,a)} = \mu^\top \abs{\vx}\,.
    \end{equation*}
\end{definition}
One can verify that $\mu$-norm satisfies triangle inequality and is a valid norm when $\mu(s,a)>0$ for all $(s,a)$; otherwise $\mu$-norm is a \emph{semi-norm} in general. Equipped with $\mu$-norm, we are interested in the following decomposition error,
\[\Big\lVert Q^\pi_{AB}-\pa{Q^\pi_A\otimes \ve + \ve \otimes Q^\pi_B} \Big\rVert_{\mu^\pi_{AB}}=\sum_{\vs,\va} \muAB(\vs,\va)\Big|Q^\pi_{AB}(\vs,\va)-(Q^\pi_A(s_A,a_A)+Q^\pi_B(s_B,a_B))\Big|\,.\]
 In other words, the state action pair $(\vs,\va)$ with higher occupancy measure $\muAB(\vs,\va)$ contributes more to the overall decomposition error. We note that the error bound in $\mu$-norm is weaker than that in $\ell_\infty$-norm. Nevertheless, a stationary distribution weighted error bound is sufficient in many practical scenarios. Similar ideas are also quite common in policy evaluation literature (\citealt{cai19neural, john96analysis, jalaj21afinite}). 

We then focus on the distance measure for Markov entanglement. To begin, we follow the same idea of $\mu$-norm and define the following distance.
\begin{definition}[$\mu$-weighted Total Variation Distance]Given probability distribution $\mu\in\R^{|S||A|}$, the $\mu$-weighted total variation distance between two transition matrices $\mP,\mP^\prime\in\R^{|S||A|\times |S||A|}$~is 
 \[\norm{\mP-\mP^\prime}_{\mu\rm{-TV}} = \sum_{(s,a)\in\gS\times\gA} \mu(s,a) D_{\rm{TV} }(\mP(\cdot,\cdot\mid s,a), \mP^\prime(\cdot,\cdot\mid s,a)) \,,\]where $D_{\rm{TV}}$ is the total variation distance between probability measures.
\end{definition}
This distance is an intuitive counterpart of the total variation distance in Definition~\ref{def: TVD}, with the maximum operator replaced by a $\mu$-weighted average. It quickly follows that $\norm{\mP-\mP^\prime}_{\mu\rm{-TV}}\leq \norm{\mP-\mP^\prime}_{\rm{TV}}$ and thus the measure of Markov entanglement w.r.t $\mu$-TV is smaller than that w.r.t TV. Analogous to the preceding subsections, we can also define the counterpart of ATV distance.

\begin{definition}[$\mu$-weighted Agent-wise Total Variation Distance]Given probability distribution $\mu\in\R^{|S|^2|A|^2}$, the $\mu$-weighted total variation distance between two transition matrices $\mP,\mP^\prime\in\R^{|S|^2|A|^2\times |S|^2|A|^2}$ w.r.t agent $A$ is defined as 
 \[\norm{\mP-\mP^\prime}_{\mu\rm{-ATV}_A} = \frac{1}{2} \sup_{\|\vx\|_\infty=1}\|\pa{\mP-\mP^\prime}(\vx\otimes \ve)\|_\mu\,. \]
\end{definition}
The $\mu$-weighted ATV distance w.r.t agent $B$ can be defined similarly. We claim that the $\mu$-weighted ATV is also a counterpart of ATV distance in Definition~\ref{def: ATVD}. This follows from the constrained optimization formulation of ATV in~\eqref{eq: ATV-optimization} where $\mu$-ATV substitutes $\mu$-norm for the original $\ell_\infty$-norm. Moreover, we have\[\norm{\mP-\mP^\prime}_{\mu\rm{-ATV}_A}\leq \frac{1}{2}\norm{\mP-\mP^\prime}_{\mu,\infty}\leq \norm{\mP-\mP^\prime}_{\mu\rm{-TV}}\,,\]
where $\norm{\mP-\mP^\prime}_{\mu,\infty}=\sup_{\|\vx\|_\infty=1}\|\pa{\mP-\mP^\prime}\vx\|_\mu$ is the induced matrix norm . This indicates ATV is still a tighter distance compared to TV distance after being weighted by the stationary distribution. We plug in $\mu$-weighted ATV to~\eqref{eq: two agent measure} and obtain the corresponding measure of Markov entanglement $E(\PAB)$ and $E_A(\PAB)$. Similar to ATV in \eqref{eq: degree of independent}, this $\mu$-weighted version of $E_A(\PAB)$ admits the following formulation 
\begin{equation}\label{eq: degree_of_indep_rho}
E_A(\PAB)\leq \min_{\mP_A} \sum_{\vs,\va} \mu^\pi_{AB}(\vs,\va) D_{\rm{TV}}\Big(\PAB(\cdot,\cdot\mid \vs,\va), \mP_A(\cdot,\cdot\mid s_A,a_A)\Big)\,.
\end{equation}
\eqref{eq: degree_of_indep_rho} substitutes the $\mu$-weighted average for the maximum operator in \eqref{eq: degree of independent}. Thus intuitively, $E_A(\PAB)$ measures \emph{how closely agent $A$ can be approximated as an independent subsystem under the stationary distribution}. 

Finally, the following theorem justifies the Q-value decomposition error in $\mu$-norm is controlled by the Markov entanglement measured using $\mu$-weighted ATV distance. 

\begin{theorem}\label{thm: rho-weighted decomp}
    Consider a two-agent MDP $\gM_{AB}$ and policy $\pi\colon \gS \to \Delta(\gA)$ with the measure of Markov entanglement $E_A(\PAB),E_B(\PAB)$ w.r.t the $\mu^\pi_{AB}$-weighted agent-wise total variation distance, it holds for each agent $i\in\{A,B\}$,
    \[\norm{\mP^\pi_i-\mP_i }_{\mu^\pi_i,\infty}\leq 2E_i(\PAB)\,.\]
    where $\mP_i$ are the optimal solutions of \eqref{eq: two agent measure} under $\mu$-weighted agent-wise total variation distance w.r.t agent $i$ respectively and $\mu^\pi_i$ is the stationary distribution of $\mP^\pi_i$. Furthermore, the decomposition error in $\mu_{AB}^\pi$-norm is bounded by the measure of Markov entanglement,
    \[\Big\lVert Q^\pi_{AB}-\pa{Q^\pi_A\otimes \ve + \ve \otimes Q^\pi_B} \Big\rVert_{\mu^\pi_{AB}} \leq \frac{4\gamma \pa{E_A(\PAB)r_{\max}^A+E_B(\PAB)r_{\max}^B}}{(1-\gamma)^2}\,.\]
\end{theorem}

This theorem serves as a $\muAB$-weighted counterpart to Theorem~\ref{thm: two-agent atv}. It measures the decomposition error using a weaker $\muAB$-norm, while the condition on $\PAB$ is also relaxed, requiring only a weighted average bound as demonstrated in \eqref{eq: degree_of_indep_rho}. These trade-offs are summarized in Table~\ref{table: distance_measure}. In the following sections, we will provide an example of how we can apply Theorem~\ref{thm: rho-weighted decomp} to bound the decomposition error for index policies in $\mu$-norm. 

\begin{table}[h]
    \centering
    \renewcommand{\arraystretch}{1.5}
    \begin{tabular}{c|c}
        \hline
        \textbf{Distance Measures of Markov Entanglement} & \textbf{Norms of Decomposition Error Bound} \\ \hline
        Total variation distance                              & \multirow{2}{*}{$\norm{\cdot}_\infty$}            \\ \cline{1-1}
        Agent-wise total variation distance                              &                                         \\ \hline
        $\mu$-weighted total variation distance                              & \multirow{2}{*}{$\norm{\cdot}_\mu$}            \\ \cline{1-1}
        $\mu$-weighted agent-wise total variation distance                              &                                         \\ \hline
    \end{tabular}\vspace{5pt}
    \caption{Different distance measures of Markov entanglement lead to decomposition error in different norms.}
    \label{table: distance_measure}
\end{table}
Finally, the major framework of the proof parallels that of Theorem~\ref{thm: two-agent atv} and further takes into account several special properties of $\mu$-norm. One noteworthy result is that the local stationary distribution $\mu^\pi_A$ turns out to be the marginal distribution of the global stationary distribution $\mu^\pi_{AB}$. That is, for all $ (s_A,a_A),\,\mu^\pi_A(s_A,a_A)=\sum_{s_B,a_B}\mu^\pi_{AB}(s_A, s_B,a_A,a_B)$. The full proof of Theorem~\ref{thm: rho-weighted decomp} is delayed to Appendix~\ref{app: proof of mu-weight}.

\section{Multi-agent Markov Entanglement}\label{sec: multi-agent results}
In quantum physics, the concept of quantum entanglement of two-party system can be well extended to multi-party system. In this section, we demonstrate a similar extension of two-agent Markov entanglement to multi-agent settings. We begin with the model of multi-agent MDPs.

Consider a standard $N$-agent MDP $\gM_{1:N}(\gS,\gA,\mP,\vr_{1:N},\gamma)$ with joint state space $\gS=\times_{i=1}^N\gS_i$ and joint action space $\gA=\times_{i=1}^N\gA_i$. For simplicity, we assume $|\gS_i|=|\gS|$ and $|\gA_i|=|\gA|$ for each agent $i$.  For agents at global state $\vs=(s_1,s_2,\ldots,s_N)$ with action $\va=(a_1,a_2,\ldots,a_N)$ taken, the system will transit to $\vs^\prime=(s_1^\prime,s_2^\prime,\ldots,s_N^\prime)$ according to transition kernel $\vs^\prime\sim\mP(\cdot\mid \vs,\va)$ and each agent $i\in[N]$ will receive its local reward $r_i(s_i,a_i)$.  The global reward $r_{1:N}$ is defined as the summation of local rewards $r_{1:N}(\vs,\va)\coloneqq \sum_{i=1}^Nr_i(s_i,a_i)$, or in vector form, 
\[\vr_{1:N}\in\R^{|\gS|^N|\gA|^N}\coloneqq \sum_{i=1}^N (\ve\otimes)^{i-1}\vr_i(\otimes \ve)^{N-i}\,.\]
We further assume the local rewards are bounded, i.e. for agent $i\in [N]$, $\abs{r_{i}(s_i,a_i)}\leq r_{\max}^i$ for all $(s_i,a_i)$. Given any global policy $\pi\colon \gS\to\Delta(\gA)$, we denote $\mP^\pi_{1:N}\in\R^{|\gS|^N|\gA|^N\times |\gS|^N|\gA|^N}$ as the transition matrix induced by $\pi$ where $P^\pi_{1:N}\pa{\vs^\prime,\va^\prime\mid \vs,\va}\coloneqq \mP\pa{\vs^\prime\mid \vs,\va}\pi\pa{\va^\prime\mid \vs^\prime}\,.$
Then the global Q-value is defined by Bellman Equation $Q^\pi_{1:N}=(\mI-\gamma \mP_{1:N}^\pi)^{-1}\vr_{1:N}$. The local Q-values follow the similar framework to Meta Algorithm~\ref{alg: meta} where each agent $i\in[N]$ fits $Q^\pi_i$ using its local observations. We then sum up local Q-values to approximate the global Q-value, i.e. 
\[
Q^\pi_{1:N}(\vs,\va)=\sum_{i=1}^NQ_i^\pi(s_i,a_i)\,.\]

To illustrate the extension, we first provide the definition of multi-party quantum entanglement.
\begin{definition}[Multi-party Quantum Entanglement]
 Consider a multi-party quantum system composed of 
$N$ subsystems, indexed by $[N]$. The joint state $\rho_{1:N}$ is \textbf{separable} if there exists $K\in\mathbb Z ^+$, probability distribution $\left\{x_i\right\}_{i\in[K]}$, and density matrices $\left\{\rho^{(j)}_{1:N}\right\}_{j\in[K]}$ such that
\[
\rho_{1:N}=\sum_{j=1}^K x_j \rho^{(j)}_1 \otimes \rho^{(j)}_2 \otimes \cdots \otimes \rho^{(j)}_N\,.\]
If there exists no such decomposition, $\rho_{1:N}$ is called \textbf{entangled}.
\end{definition}
Analogically, we define the Multi-agent Markov Entanglement,

\begin{definition}[Multi-agent Markov Entanglement]
    Consider a $N$-agent Markov system $\gM_{1:N}$ and policy $\pi\colon \gS \to \Delta(\gA)$, the agents are \textbf{separable} under policy $\pi$ if there exists $K\in\mathbb Z ^+$, measure $\{x_j\}_{j\in[K]}$ satisfying $\sum_{j=1}^Kx_j=1$, and transition matrices $\left\{\mP^{(j)}_{1:N}\right\}_{j\in[K]}$ such that 
    \[\mP^\pi_{1:N} = \sum_{j=1}^K x_j\mP_1^{(j)}\otimes \mP_2^{(j)}\otimes\cdots \otimes\mP_N^{(j)} \,.\]
    If there exists no such decomposition, the agents are \textbf{entangled}.
\end{definition}
It readily follows that we can similarly define the measure of multi-agent Markov entanglement as
\begin{equation}\label{eq: Multi-agent Markov Entanglement}
    E(\PN) = \min_{\mP\in\SEP } d(\mP^\pi_{1:N}, \mP)\,,
\end{equation}
where $\SEP$ is set of all separable $N$-agent transition matrices and $d(\cdot,\cdot)$ is some distance measure.

We note that multi-agent Markov entanglement retains the core idea that a separable system can be expressed as \emph{a linear combination of independent subsystems}. Furthermore, it is not surprising that we can derive a similar result for multi-agent MDPs concerning exact value decomposition, analogous to Theorem~\ref{thm: mixed_state}, and general decomposition error in Theorem~\ref{thm: two-agent value decomp}, \ref{thm: two-agent atv}, and~\ref{thm: rho-weighted decomp}. We provide one extension of Theorem~\ref{thm: rho-weighted decomp} below and delay the full results to Appendix~\ref{app: multi-agent}. 
\begin{theorem}\label{thm: multi-agent-rho}
    Consider a $N$-agent MDP $\gM_{1:N}$ and policy $\pi\colon \gS \to \Delta(\gA)$ with the measure of Markov entanglement $E_i(\PN)$ w.r.t the $\mu^\pi_{1:N}$-weighted agent-wise total variation distance, it holds for any agent $i\in[N]$,
    \[\norm{\mP^\pi_i-\mP_i }_{\mu^\pi_{i},\infty}\leq 2E_i(\PN)\,.\]
    where $\mP_i$ is the optimal solution of \eqref{eq: Multi-agent Markov Entanglement} and $\mu_i^\pi$ is the stationary distribution of the projected transition $\mP^\pi_i$. Furthermore, the decomposition error in $\mu^\pi_{1:N}$-norm is bounded by the measure of Markov entanglement,
    \[\norm{ Q^\pi_{1:N}(\vs,\va)-\sum_{i=1}^N Q^\pi_i(s_i,a_i) }_{\mu^\pi_{1:N}} \leq \frac{4\gamma \pa{\sum_{i=1}^NE_i(\PN)r_{\max}^i}}{(1-\gamma)^2}\,.\]
\end{theorem}

\section{Applications of Markov Entanglement}
This section discusses both theoretical and practical applications of Markov entanglement. We first show how specific MDP structures simplify entanglement analysis and produce sharp decomposition error bounds. We then demonstrate how Markov entanglement serves as an efficient test criteria of value decomposition for practitioners. Finally, we numerically study two important multi-agent scenarios: a synthetic restless multi-armed bandit model and a ride-hailing simulator.

\subsection{(Weakly-)coupled MDPs}
Weakly-coupled MDPs (WCMDP) are a rich class of multi-agent model that capture many real-world applications such as supply chain management, queuing network and resource allocations (\citealt{daniel08relaxation, david23on, shar2023weakly}). Compared to general multi-agent MDP, WCMDP further ensures each agent follow its local transition while the agents' actions are coupled with each other. Formally, 
\begin{definition}[Weakly-coupled MDPs]
    An $N$-agent MDP $\gM_{1:N}(\gS,\gA,\mP,\vr_{1:N},\gamma)$ is a weakly-coupled MDP if 
    \begin{itemize}
        \item Each agent has local transition kernel $\mP_i$ such that $\forall \vs,\va,\vs^\prime, P(\vs^\prime\mid \vs, \va)=\prod_{i=1}^NP_i(s_i^\prime\mid s_i,a_i)$.
        \item At global state $\vs$, agents' joint actions $\va$ are subject to $m$ coupling constraints $\sum_{i=1}^N \vd_i(s_i,a_i)\leq \vb\in \R^m$ and $\vd_i\colon\brk{(s_i,a_i)c\colon s_i\in\gS_i,a_i\in\gA(s_i)}\to\R^m$.
    \end{itemize}   
\end{definition}

We then demonstrate that this weakly-coupled structure can further refine the analysis of Markov entanglement measure.

\begin{proposition}\label{prop: entanglement of policy}
Consider a $N$-agent weakly-coupled MDP $\gM_{1:N}(\gS,\gA,\mP,\vr_{1:N},\gamma)$. Given any policy $\pi\colon \gS \to \Delta(\gA)$ with measure of Markov entanglement $E_i(\mP_{1:N}^\pi)$ w.r.t the $\mu^\pi_{1:N}$-weighted agent-wise total variation distance, it holds for $i\in[N]$,
\[E_i(\mP_{1:N}^\pi)\leq \min_{\pi^\prime}\frac{1}{2}\sum_{\vs} \mu^\pi_{1:N}(\vs)\sum_{a_i}\Big|\pi(a_i\mid \vs)-\pi^\prime(a_i\mid s_i)\Big|\,,\]
where $\pi^\prime:\gS_i\to\gA_i$ is any local policy for agent $i$.
\end{proposition}

\begin{proof}[Proof of Proposition~\ref{prop: entanglement of policy}] We demonstrate the proof for two-agent WCMDP and the generalization to multi-agent WCMDP is straightforward. Consider $\mP_A^{\pi^\prime}$ be the transition of agent $A$ under local policy $\pi^\prime$. We focus on agent $A$ \begin{align*}
        E_A(\PAB)
        &\leq\frac{1}{2}\sum_{\vs,\va} \muAB(\vs,\va)\sum_{s_A^\prime,a_A^\prime} \abs{\sum_{s_B^\prime}P^\pi_{AB}(\vs^\prime,a_A\mid \vs,\va) - P_{A}^{\pi^\prime}(s_A^\prime\mid s_A,a_A)\pi^\prime(a_A^\prime\mid s_A^\prime)}\\
        &\stackrel{(i)}{=}\frac{1}{2}\sum_{\vs,\va} \muAB(\vs,\va)\sum_{s_A^\prime,a_A^\prime} \abs{\sum_{s_B^\prime}P^\pi_{AB}(\vs^\prime,a_A\mid \vs,\va) - \sum_{s_B^\prime}P(\vs^\prime\mid \vs,\va)\pi^\prime(a_A^\prime\mid s_A^\prime)}\\
        &=\frac{1}{2}\sum_{\vs,\va} \muAB(\vs,\va)\sum_{s_A^\prime,a_A^\prime} \abs{ \sum_{s_B^\prime}P(\vs^\prime\mid \vs,\va)\pa{\pi(a_A^\prime\mid \vs^\prime)-\pi^\prime(a_A^\prime\mid s_A^\prime)}}\\
        &\leq \frac{1}{2}\sum_{\vs,\va}\muAB(\vs,\va)\sum_{\vs^\prime} P(\vs^\prime\mid \vs,\va) \sum_{a_A^\prime} \abs{\pi(a_A^\prime\mid \vs^\prime)-\pi^\prime(a_A^\prime\mid s_A^\prime)}\\
        &\stackrel{(ii)}{=}\frac{1}{2} \sum_{\vs^\prime} \muAB(\vs^\prime) \sum_{a_A^\prime} \abs{\pi(a_A^\prime\mid \vs^\prime)-\pi^\prime(a_A^\prime\mid s_A^\prime)}\,.
    \end{align*}
    where $(i)$ follows from the transition structure of weakly coupled MDP $P(\vs^\prime\mid \vs,\va)=P(s^\prime_A\mid s_A,a_A)\cdot P(s^\prime_B\mid s_B,a_B)$; and $(ii)$ comes from the fact that $P^\pi(\vs^\prime\mid \vs)=\sum_{\va}\pi(\va\mid \vs)P(\vs^\prime\mid \vs,\va)$ and $\sum_{\vs}\mu^\pi(\vs)P^\pi(\vs^\prime\mid \vs)=\mu^\pi(\vs^\prime)$.
\end{proof}
Proposition~\ref{prop: entanglement of policy} establishes an upper bound for the Markov entanglement measure in WCMDP. Intuitively, this bound characterizes \emph{how agent $i$ can be viewed as making independent decisions}. It takes advantage of the local transition structure, thereby shaving off the transition term. In the next subsection, we will demonstrate how this can ease our analysis of the Markov entanglement measure.

Moreover, we observe that Proposition~\ref{prop: entanglement of policy} does not depend on the linear coupling constraint $\sum_{i=1}^N \vd_i(s_i,a_i)\leq \vb$. Instead, it applies generally to multi-agent MDPs with arbitrary coupling, provided agents adhere to local transition kernels.

\subsubsection{Index Policies are Asymptotically Separable}\label{sec: index policy weak entangle}

We further dive into the more structured multi-agent MDPs and introduce the following model of Restless Multi-Armed Bandit (RMAB), a special instance of weakly coupled MDP which is also widely used in operations research literature (\citealt{Whittle88restless, weber90on, gast23exponential, zhang2021restless,zhang2022near}).

\begin{definition}[Restless Multi-Armed Bandit]
A Restless Multi-Armed Bandit is an $N$-agent WCMDP that further satisfies
\begin{itemize}
    \item There are two available actions for each agent: $0$ for idle and $1$ for activate.
    \item Agent are homogeneous, i.e. with the same local state space $\gS$,\footnote{We abuse the notation $\gS$ to refer to the local state space in the context of RMAB since agents are homogeneous.} local transition $\{\mP_0,\mP_1\}$  and reward $\{\vr_0,\vr_1\}$ bounded  by $r_{\max}$. 
    \item $M\leq N$ agents will be activated at each timestep and other agents are left idle.
\end{itemize}
\end{definition}
In other words, RMAB is WCMDP with two actions and homogeneous agents that are coupled under constraint $\sum_{i=1}^N a_i=M$ at any global state $s_{1:N}$. This coupling of agents is often referred to budget constraint. For example, healthcare platform can reach out only a fraction of patients at a time due to the cost budget of interventions (\citealt{baek2023policy}).

In RMAB, arguably the most popular and classical policy is the index policy, where the decision maker activates agents based on some priority of their local states. We formally define   

\begin{definition}[Index Policy]
    There exists a priority index $\nu_s$ for each local state $s$. The decision maker will always activate agents in the descending order of the priority until the budget constraint $M$ is met. Ties are resolved fairly via uniform random sampling of agents at the same state.
\end{definition}

The index policies trace back to the well-known optimal Gittins Index (\citealt{weber92on}) and asymptotic optimal Whittle Index (\citealt{Whittle88restless, weber90on,gast23exponential}). More recent work generalizes Whittle Index to fluid based index policies (\citealt{Verloop2016AsymptoticallyOP, gast24linear}). However, computing the best index policy typically requires the knowledge of system transition a prior. As an alternative, \cite{qin20ride, xabi24better, baek2023policy, nakhleh21neu, wang2023optimistic, avrachenkov2022whittle} apply data-driven method to optimize the index policy. Among these work, \cite{qin20ride, xabi24better, baek2023policy} report great empirical success in industry-scale implementations. Understanding the mystery behind such success calls for a theory for general index policies. We then present our main result for this subsection

\begin{theorem}\label{thm: RMAB}
    Consider an $N$-agent restless multi-armed bandit. For any index policy satisfying mild technique conditions, it is asymptotically separable. Furthermore, there exists constant $C$ independent of $N$, such that for any agent $i\in [N]$, the measure of Markov entanglement $E_i(\mP^\pi_{1:N})$ w.r.t the $\mu^\pi_{1:N}$-weighted agent-wise total variation distance is bounded, \[E_i(\mP^\pi_{1:N})\leq \frac{C}{\sqrt{N}}\,.\]
\end{theorem}
Theorem~\ref{thm: RMAB} requires two standard technical conditions for index policies: non-degenerate and uniform global attractor property (UGAP), which restrict chaotic behavior in asymptotic regime and will be detailed in Appendix~\ref{app: index}. We note here these two assumptions are also used in many previous theoretical work on index policies (\citealt{weber90on,Verloop2016AsymptoticallyOP,gast23exponential,gast24linear}).

Theorem~\ref{thm: RMAB} justifies index polices are asymptotically separable under standard technical assumptions. Combined with Theorem~\ref{thm: multi-agent-rho}, we obtain the sublinear decomposition error for index policies
\begin{corollary}\label{coro: index_error}
    Consider an $N$-agent restless multi-armed bandit. For any index policy satisfying: (i) non-degenerate, (ii) UGAP, there exists constant $C$ independent of $N$ such that
    \[\norm{ Q^\pi_{1:N}(\vs,\va)-\sum_{i=1}^N Q^\pi_i(s_i,a_i) }_{\mu^\pi_{1:N}} \leq \frac{4C\gamma \sqrt{N}r_{\max}}{(1-\gamma)^2}\,.\]
\end{corollary}
This sublinear error result explains why the value decomposition in \cite{qin20ride, xabi24better, baek2023policy}
manages to effectively approximate the global value function in large-scale practical applications. It also justifies Meta Algorithm~\ref{alg: meta} as an effective approach to evaluating or comparing index policies.

\subsubsection{Proof Sketch of Theorem~\ref{thm: RMAB}}
We provide an overview of the proof and delay the full version to Appendix~\ref{app: index}.

To begin, we consider the system configuration $\vm\in\Delta^{|\gS|}$ where $\vm_s=\frac{1}{N}\sharp\{\rm{Agents\ in\ state\ }s\}$ is the proportion of agents in state $s$. When $N\to\infty$, the transition between configurations can be viewed as deterministic under the index policy and $\vm$ approaches its mean-field fixed-point $\vm^\ast$. Furthermore, in this mean-field limit, each agent's local transition will only depend its local state. As a result, the system will de-couple and become separable as $N\to\infty$.

To formalize this intuition, we introduce the following lemma that connects Markov entanglement measure with the mean-field analysis
\begin{lemma}
For any index policy satisfying the same condition as Theorem~\ref{thm: RMAB}, the measure of Markov entanglement w.r.t $\mu^\pi_{1:N}$-weighted ATV distance is bounded by the deviation from the mean-field configuration, i.e. for any agent $i\in [N]$,
    \[E_i(\PN)\leq |\gS|^2\cdot\E\br{\|\vm-\vm^\ast\|_\infty}\,,\]
    where the expectation is taking over the stationary distribution $\vm\sim \mu^\pi_{1:N}$.
\end{lemma}
This lemma builds upon Proposition~\ref{prop: entanglement of policy}. We thus focus on the deviation from $\vm$ to $\vm^\ast$. We extend the concentration analysis from~\cite{gast23exponential, gast24linear} to derive a new stability bound for the RHS. Specifically, we finishing the proof via demonstrating the deviation decays at the rate $\gO(1/\sqrt{N})$.

\subsection{Efficient Verification of Value Decomposition}
For practitioners, verifying the feasibility of value decomposition remains a significant challenge. Typically, they have to rely on indirect methods—for instance, evaluating policies derived from value decomposition in real-world settings or simulation environments, as seen in \cite{baek2023policy, qin20ride, xabi24better}. While these policies often outperform baselines, it is unclear whether the learned decompositions accurately approximate the true global values. Moreover, real-world verifications are often prohibitively expensive. For example, the Lyft experiment (\citealt{xabi24better}) required a complex time-split design and posed substantial software engineering challenges in reliability and stability. Similarly, \cite{baek2023policy, qin20ride} depended on extensive offline data collection to construct viable simulation environments.


As a solution, Markov entanglement offers a simple and efficient way to empirically test whether value decomposition can be safely applied. Recall the decomposition error in $\muAB$-norm is controlled by the measure of Markov entanglement w.r.t $\muAB$-weighted ATV distance. Thus it suffices to estimate $E_A(\PAB)$. Specifically, according to \eqref{eq: degree_of_indep_rho}, we have
\begin{align}
    E_A(\PAB)
    &\leq \min_{\mP_A} \sum_{\vs,\va} \rho^\pi_{AB}(\vs,\va) D_{\rm{TV}}\Big(\PAB(\cdot,\cdot\mid \vs,\va), \mP_A(\cdot,\cdot\mid s_A,a_A)\Big)\nonumber\\
    &\approx \min_{\mP_A} \frac{1}{2T}\sum_{t=1}^T \sum_{s_A^\prime,a_A^\prime} \Big|\PAB(s_A^\prime,a_A^\prime\mid \vs^t,\va^t)-\mP_A(s_A^\prime,a_A^\prime\mid s_A^t,a_A^t) \Big|\label{eq: monte-carlo}
\end{align}
In other words, we can apply a Monte-Carlo estimation for estimating $E_A(\PAB)$. Notice~\eqref{eq: monte-carlo} is essentially a \emph{linear programming} for $\mP_A$ and this optimization can be solved distributively at each $(s_A,a_A)$ pair, which enables efficient solutions. Moreover, \eqref{eq: monte-carlo} only requires the knowledge of one-step transition $\PAB(s_A^\prime,a_A^\prime\mid \vs^t,\va^t)$, which can often be easily calculated or simulated. For (weakly-)coupled MDPs, we can apply Proposition~\ref{prop: entanglement of policy} and further eliminate the transition term,
\begin{align}
    E_A(\PAB)
    &\leq \min_{\pi^\prime}\frac{1}{2}\sum_{\vs} \mu^\pi_{AB}(\vs)\sum_{a_A}\Big|\pi(a_A\mid \vs)-\pi^\prime(a_A\mid s_A)\Big|\nonumber\\
    &\approx \min_{\pi^\prime}\frac{1}{2T}\sum_{t=1}^T \sum_{a_A}\Big|\pi(a_A\mid \vs^t)-\pi^\prime(a_A\mid s_A^t)\Big|\,.\label{eq: monte-carlo-wcmdp}
\end{align}
Compared to \eqref{eq: monte-carlo}, \eqref{eq: monte-carlo-wcmdp} only calls for the knowledge of the global policy $\pi$. Together, \eqref{eq: monte-carlo} and \eqref{eq: monte-carlo-wcmdp} enable efficient estimation of $E_A(\PAB)$ via empirical simulation. These ideas can also be easily extended to $N$-agent MDPs. 

\paragraph{Worst-case Error Bound} We also emphasize the decomposition error bound in Theorem~\ref{thm: rho-weighted decomp} guarantees worst-case performance. For instance, an MDP may exhibit high entanglement yet incur small decomposition error (e.g., $\vr_A=\vr_B=\vzero$). However, low Markov entanglement strictly ensures a small decomposition error. This monotonic property establishes Markov entanglement as a conservative verification metric. Practitioners can thus confidently apply value decomposition whenever the system exhibits low entanglement.


\subsubsection{Numerical Simulation I: Restless Multi-armed Bandit}
We first empirically study the value decomposition for the index policy on a circulant RMAB benchmark \cite{avrachenkov2022whittle, zhang2022near,biswas21learn, fu19towards} that has $4$ different states each local agent. As a result, the global state space scales as large as $4^{1800}> 10^{1000}$ for $N=1800$ agents. The specific transitions and rewards are introduced in Appendix~\ref{app: simulation}. For each RMAB instance, we sample a trajectory of length $T=5N$ and use the collected data to i) solve~\eqref{eq: monte-carlo-wcmdp} to estimate the measure of Markov entanglement; ii) train local Q-value decomposition. It quickly follows from Figure~\ref{fig: RMAB_Exp_main_text}:

\begin{figure}[h]
\centering
\begin{tabular}{l l }
        \includegraphics[height=.35\linewidth]{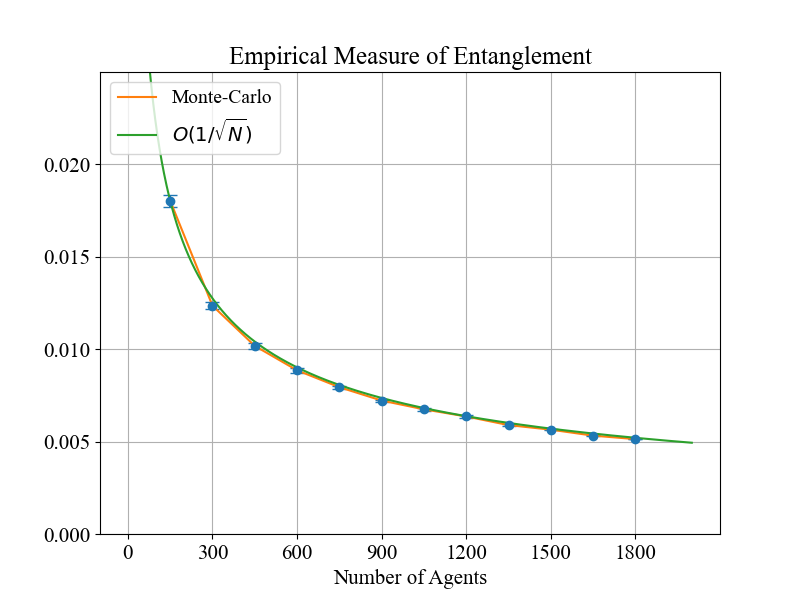}
        &
        \includegraphics[height=.35\linewidth]{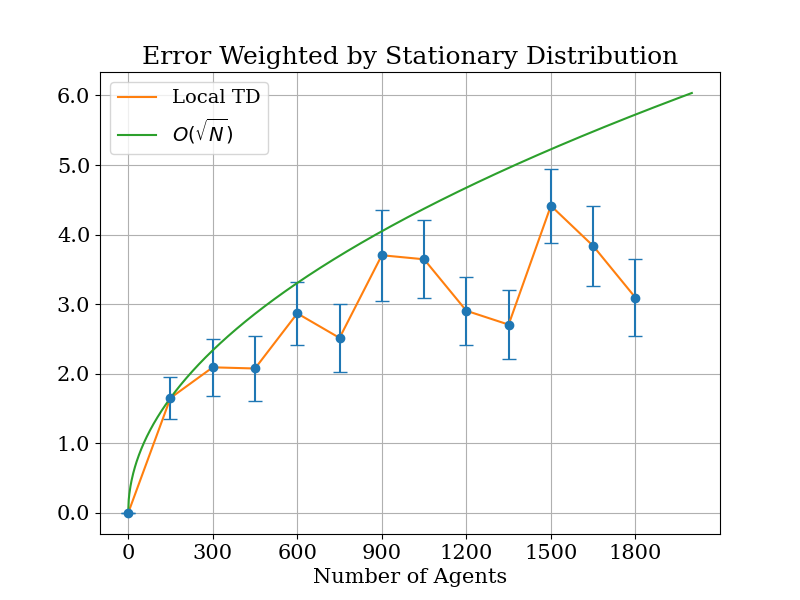}
    \end{tabular}
    \caption{Circulant RMAB under an index policy. \emph{Left:} empirical estimation of Markov entanglement $E_{1}(\PN)$. \emph{Right:} $\mu$-weighted decomposition error. }
    \label{fig: RMAB_Exp_main_text}
\end{figure}
The estimated Markov entanglement decays as $\gO(1/\sqrt N)$ in the left panel, consistent with theoretical predictions. This also implies a low decomposition error scaling of $\gO(\sqrt N)$, as seen in the right panel. Furthermore, the simulated trajectory has a length of $T=5N$ while the global state space has size $|S|^N$, showing both entanglement estimation and local Q-value decomposition sample-efficient.

\subsubsection{Numerical Simulation II: A Ride-hailing Simulator}
Finally, we study the Markov entanglement in ride-hailing (\citealt{xabi24better, qin20ride}), via a simulator built on NYC yellow cab data (\citealt{NYCTLCData}). The ride-hailing setting presents significantly greater complexity than RMAB. Most notably, a set of exogenous orders arrives at each timestep, and drivers are matched to these orders according to specific dispatching rules. Particularly, with $N$ drivers, we sample $0.1N$ orders at each timestep. Matched drivers transition to new positions based on their assigned orders, while idle drivers may also relocate autonomously (\citealt{han22real}). To address this challenge, we extend the Markov entanglement framework to accommodate exogenous orders (see Appendix~\ref{app: exog_order}) and derive efficient estimators for both the Markov entanglement measure and local value functions in this setting. Simulation results are exhibited below with more details in Appendix~\ref{app: exp_RH}.

\begin{figure}[h]
\centering
\begin{tabular}{l l }
        \includegraphics[height=.35\linewidth]{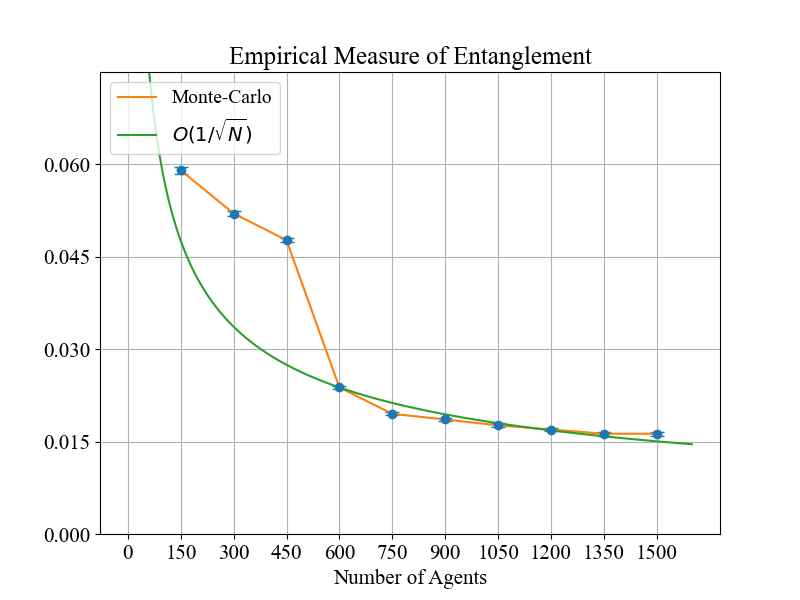}
        &
        \includegraphics[height=.35\linewidth]{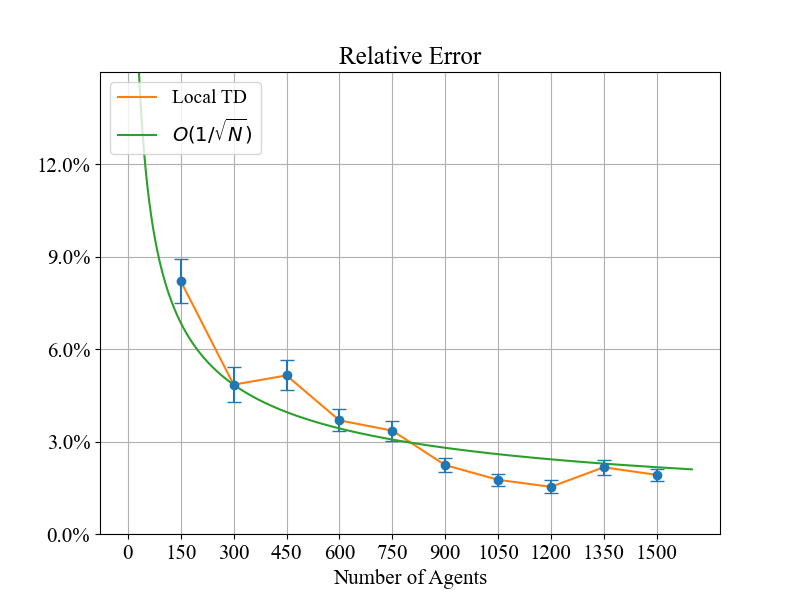}
    \end{tabular}
    \caption{A ride-hailing simulator. \emph{Left:} empirical estimation of Markov entanglement $E_{1}(\PN)$. \emph{Right:} $\mu$-weighted decomposition error divided by the global value $\|Q^\pi_{1:N}\|_\mu$. }
    \label{fig: RH_Exp_main_text}
\end{figure}

Perhaps surprisingly, despite the complexity of ride-hailing system, its Markov entanglement measure remains generally small and decays as $\gO(1/\sqrt N)$ for large $N$. We conjecture that this decay arises because the ride-hailing simulator converges to its mean-field limit as $N$ grows, which also exhibits asymptotic separability. More rigorous theoretical analysis is left for future work. Furthermore, the decomposition error relative to the true global Q-values decays at rate $\gO(1/\sqrt N)$, becoming negligibly small ($\leq3\%$) for large $N$. These findings help explain the empirical success of value decomposition methods previously observed in ride-hailing applications.

\section{Conclusion}
This paper established the mathematical foundation of value decomposition in MARL. Drawing inspiration from quantum physics, we propose the idea of Markov entanglement and prove that it serves as a sufficient and necessary condition for the exact value decomposition. We further characterize the decomposition error in general multi-agent MDPs through the measure of Markov entanglement. As application examples, we prove widely-used index policies are asymptotically separable and suggest practitioners using Markov entanglement as a proxy for estimating the effectiveness of value decomposition.

Reinforcement learning and quantum physics have been two well-established yet largely separate fields. We hope our study opens an interesting connection between them, allowing concepts/techniques developed in one field to benefit the other.


\bibliographystyle{informs2014}
{
\bibliography{reference}
}

\newpage
\begin{appendices}
\tableofcontents
\crefalias{section}{appendix}
\crefalias{subsection}{appendix}

\section{Linear Algebra with Tensor Product}\label{app: linaer algebra}
We briefly introduce the basic properties of tensor product or Kronecker product. Let $\mA\in\R^{m_1\times n_1}, \mB\in\R^{m_2\times n_2}$, then\[\mA\otimes\mB=\left[\begin{array}{cccc}
a_{11} \mB & a_{12} \mB & \ldots & a_{1 n_1} \mB \\
a_{21} \mB & a_{22} \mB & \ldots & a_{2 n_1} \mB \\
\ldots & \ldots & \ldots & \ldots \\
a_{m_1 1} \mB & a_{m_1 2} \mB & \ldots & a_{m_1 n_1} \mB
\end{array}\right]\in\R^{m_1m_2\times n_1n_2}\,.\]Tensor product satisfies the following basic properties,
\begin{itemize}
    \item [\textbf{1. Bilinearity}] For any matrix $\mA,\mB,\mC$ and constant $k$, it holds $k(\mA\otimes\mB)=(k\mA)\otimes\mB=\mA\otimes(k\mB)$, $(\mA+\mB)\otimes \mC=\mA\otimes\mC+\mB\otimes\mC$, and $\mA\otimes (\mB+\mC)=\mA\otimes\mB+\mA\otimes\mC$.
    \item [\textbf{2. Mixed-product Property}] For any matrix $\mA,\mB,\mC,\mD$, if $\mA\mC$ and $\mB\mD$ form valid matrix product, then $(\mA\otimes\mB)(\mC\otimes\mD)=(\mA\mC)\otimes(\mB\mD)$.
\end{itemize}
\section{Decompose value functions}\label{app: decomp_value}
Compared to the decomposition of Q-value, the value function further requires the reward to be \emph{state-dependent}. To illustrate, notice by Bellman equation, 
\[V^\pi_{AB}=(\mI-\gamma\mP^\pi_{AB})^{-1}\vr_{AB}^\pi\,,\]where we abuse notation and denote $P^\pi_{AB}(\vs^\prime\mid \vs)=\sum_{\va}\pi(\va\mid \vs)P(\vs^\prime\mid \vs,\va)$ and reward $r^\pi_{AB}(\vs)=\sum_{\va}\pi(\va\mid \vs)r_{AB}(\vs,\va)$. A key subtlety arises because $\vr^\pi_{AB}$ may not be decomposable—even when $\vr_{AB}$ is decomposable—unless the reward $\vr_{AB}$ is state-dependent. Consequently, we cannot directly apply the "absorbing" equation as in the proof of Theorem~\ref{thm: mixed_state}. 

On the other hand, Q-value decomposition bypasses the state-dependence assumption and provides a stronger condition that directly implies value function decomposition. As a result, while learning local value functions may seem more intuitive, we recommend learning local Q-values instead and using them to approximate the global value function.
\section{Necessity of Negative Coefficients}\label{app: neg_coeff}
In section~\ref{sec: neg_coeff}, we discuss that compared to quantum entanglement, Markov entanglement does not require coefficients $\vx\geq0$. Particularly, we will provide an instance $\mP$ that lies in $\gP_{\textrm{SEP}}$ but not $\gP_{\textrm{SEP}}^+$.

Consider the following basis
\[
\begin{gathered}
\mE_{00}=\left(\begin{array}{ll}
1 & 0 \\
1 & 0
\end{array}\right), \quad \mE_{01}=\left(\begin{array}{ll}
1 & 0 \\
0 & 1
\end{array}\right), \quad \mE_{10}=\left(\begin{array}{ll}
0 & 1 \\
1 & 0
\end{array}\right), \quad \mE_{11}=\left(\begin{array}{ll}
0 & 1 \\
0 & 1
\end{array}\right) 
\end{gathered}
\]
And the corresponding transition matrix we provide is
\[\mP=\left(\begin{array}{cccc}
0.5 & 0 & 0 & 0.5 \\
0.5 & 0 & 0 & 0.5 \\
0.5 & 0 & 0 & 0.5 \\
0 & 0.5 & 0.5 & 0
\end{array}\right)=\frac{1}{2} \mE_{00} \otimes \mE_{00}+\frac{1}{2} \mE_{10} \otimes \mE_{11}+\frac{1}{2} \mE_{11} \otimes \mE_{10}-\frac{1}{2} \mE_{10} \otimes \mE_{10}\]
One can also verify $\mP$ can not be represented by the convex combination of tensor products of these basis.
\section{Proof of Theorem~\ref{thm: necessary_condition}}\label{app: proof_thm_2}
We provide the full proof of Theorem~\ref{thm: necessary_condition} in this section.
\paragraph{Step 1: Characterize the Orthogonal Complement.}{}

To start with, we consider the smallest subspace containing all transition matrices $\Omega_P\coloneqq\rm{span}(\mP)$ where $\mP$ are the set of all transition matrices in $\R^{m\times m}$. We then study the dimension of $\Omega_P$.
\begin{lemma}
    The dimension of $\Omega_P$ is $\rm{dim}(\Omega_P)=m^2-m+1$.
\end{lemma}
\begin{proof}
Let $\mZ_{ij}\in\R^{m\times m}$ such that 
\[\mZ_{ij}(a,b)=\left\{\begin{matrix}
  1& (a=i\wedge b=j)\vee (a=b)\\
  0& o.w.
\end{matrix}\right.\,.\]
    One basis for all transition matrices is given by $\{\mZ_{ij}\}_{i,j\in[m]}$ whose cardinarlity is $m^2-m+1$.
\end{proof}
Let $\Omega_{P^{\otimes2}}\coloneqq\rm{span}(\mP_1\otimes\mP_2)$ be the minimal subspace containing all separable transition matrices. It quickly follows that 
\[\rm{dim}(\Omega_{P^{\otimes2}})=(\rm{dim}(\Omega_P))^2\,.\]
We then construct the orthogonal complement of $\Omega_{P^{\otimes2}}$ under Frobenius inner product. Let $\{\varepsilon_j\}_{j\in[m-1]}$ be a set of vector in $\R^m$ such that $\varepsilon_j=(1,0,\ldots,0,-1,0,\ldots,0)^\top$ with the first element $1$ and $j+1$-th element $-1$. Notice that \[\Tr\pa{\ve\varepsilon_j^\top\mP}=\Tr\pa{\varepsilon_j^\top\mP\ve}=0\,,\]for all $\varepsilon_j$. Consider the following subspace\[\Omega^\prime=\brk{\sum_{j=1}^{m-1}\pa{\varepsilon_j \ve^\top}\otimes \mW^1_j + \sum_{j=1}^{m-1}\mW^2_j\otimes\pa{\varepsilon_j \ve^\top} \mid W^1_{1:j},W^2_{1:j}\in\R^{m\times m}}\,.\]
We then show $\Omega^\prime$ is exactly the orthogonal complement of $\Omega_{P^{\otimes2}}$. First, notice that 
\[\rm{dim}(\Omega^\prime) = 2(m-1)m^2-(m-1)^2\,.\] and thus $\rm{dim}(\Omega^\prime) + \rm{dim}(\Omega_{P^{\otimes2}})=m^4$. Moreover, one can verify for any $\mX\in\Omega_{P^{\otimes2}}$ and $\mY\in \Omega^\prime$, $\Tr(\mX^\top\mY)=0$. As a result, it holds \[\Omega^\prime=\Omega_{P^{\otimes2}}^\bot\,.\]

\paragraph{Step 2: Connection to ``Inverse"}{}

The decomposition of Q-value ultimately concerns with the properties of $(\mI-\gamma \mP^\pi_{AB})^{-1}$. The following lemma bridges this gap.

\begin{lemma}\label{lem: connection to inverse}
    Given any transition matrix $\mP$ and $\gamma >0$, $\mP$ is separable if and only if $(1-\gamma)(\mI-\gamma \mP)^{-1}$ is separable.
\end{lemma}
\begin{proof}
    ($\Rightarrow$) One can verify that $(\mI-\gamma\mP) \ve=(1-\gamma)\ve$, which implies $(1-\gamma)(\mI-\gamma \mP)^{-1}$ is a transition matrix. Moreover, $(1-\gamma)(\mI-\gamma \mP)^{-1}=(1-\gamma)\sum_{i=0}^\infty (\gamma\mP)^i$ falls in $\Omega_{P^{\otimes2}}$ as $\mP\in \Omega_{P^{\otimes2}}$.

    ($\Leftarrow$) This side is more involved. Denote $\mU\coloneqq(1-\gamma)(\mI-\gamma \mP)^{-1}$. Then if the spectral radius $\rho(\mI-\mU)<1$, then \[U^{-1}=\pa{\mI-\pa{\mI-\mU}}^{-1}=\sum_{i=0}^\infty (\mI-\mU)^i\in\Omega_{P^{\otimes2}}\,.\]
    This implies $U^{-1}=\frac{1}{1-\gamma} (\mI-\gamma \mP)\in\Omega_{P^{\otimes2}}$ and thus $\mP\in\Omega_{P^{\otimes2}}$, finishing the proof. It then suffices to show $\rho(\mI-\mU)<1$. Notice that
    \begin{align*}
        \lambda_i(\mI-\mU)=1-\lambda_i(\mU)=1-\frac{1-\gamma}{\lambda(\mI-\gamma\mP)}=1-\frac{1-\gamma}{1-\gamma \lambda_i(\mP)}\,.
    \end{align*}
    Let $\lambda_i(\mP)=a+bi$ and taking modulus for both side
    \begin{align*}
        \abs{\lambda_i(\mI-\mU)}&=\abs{\frac{\gamma-\gamma\lambda_i(\mP)}{1-\gamma \lambda_i(\mP)}}\\
        &= \sqrt{\frac{\gamma^2(1-a)^2+\gamma^2b^2}{(1-\gamma a)^2+\gamma^2b^2}}\\
        &=\sqrt{1+\frac{(1-\gamma)(2a\gamma-\gamma-1)}{(1-\gamma a)^2+\gamma^2b^2}}\\
        &\leq \sqrt{1-\frac{(1-\gamma)^2}{(1-\gamma a)^2+\gamma^2b^2}}<1\,.
    \end{align*}
    We conclude the proof given $\rho(\mI-\mU)=\max_i\abs{\lambda_i(\mI-\mU)}<1$.
\end{proof}

\paragraph{Step 3: Put it together}{}

By Lemma~\ref{lem: connection to inverse}, if $\mP^\pi_{AB}$ is entangled, then $(1-\gamma)(\mI-\gamma \mP^\pi_{AB})^{-1}$ is also entangled. Then there exists $\mY\in \Omega^\prime \ne \vzero$ such that $\Tr(\mY^\top(\mI-\gamma \mP^\pi_{AB})^{-1})\ne 0$. We apply singular value decomposition to all $W^1_{1:j},W^2_{1:j}$ and conclude there exists some $j$ and $\vu,\vv\in\R^m$ such that either $\Tr(\pa{ \ve\varepsilon_j^\top}\otimes \pa{\vv\vu^\top} (\mI-\gamma \mP^\pi_{AB})^{-1})\ne0$ or $\Tr( \pa{\vv\vu^\top}\otimes \pa{\ve\varepsilon_j^\top}(\mI-\gamma \mP^\pi_{AB})^{-1})\ne0$. We assume the former without loss of generality, it holds
\[(\varepsilon_j^\top\otimes\vu^\top)(\mI-\gamma \mP^\pi_{AB})^{-1} (\ve\otimes\vv)\ne 0\,.\]
Now set $\vr_A=\vzero$ and $\vr_B=\vv$. Since $Q^\pi_{AB}$ is decomposable, there exists some local function $Q_A,Q_B$ such that
\[(\mI-\gamma \PAB)^{-1} (\ve\otimes \vv)=Q_A(\vzero)\otimes\ve + \ve\otimes Q_B(\vv)\,.\]
Left multiply by $(\varepsilon_j^\top\otimes\vu^\top)$, we have 
\begin{align*}
    (\varepsilon_j^\top\otimes\vu^\top)(\mI-\gamma \mP^\pi_{AB})^{-1} (\ve\otimes\vv)=(\varepsilon_j^\top\otimes\vu^\top)(Q_A(\vzero)\otimes\ve)\ne0\,,
\end{align*}
Then set $\vr_A=\vzero$ and $\vr_B=-\vv$, we can similarly derive
\begin{align*}
    -(\varepsilon_j^\top\otimes\vu^\top)(\mI-\gamma \mP^\pi_{AB})^{-1} (\ve\otimes\vv)=(\varepsilon_j^\top\otimes\vu^\top)(Q_A(\vzero)\otimes\ve)\ne0\,,
\end{align*}
This gives use $(\varepsilon_j^\top\otimes\vu^\top)(Q_A(\vzero)\otimes\ve)=0$, which is a contradiction.

\section{Decomposition via general functions}\label{app: decomp_general}
Entangled $\mP$ precludes the local decomposition with local value functions, but may admit decompositions with more general functions. Consider \[\mP=\frac{1}{4}\left(e e^{\top}\right) \otimes\left(e e^{\top}\right)+\delta\left(\epsilon e^{\top}\right) \otimes\left(e \epsilon^{\top}\right)\,,\]
where $e=[1,1], \epsilon=[1-1]$. Clearly such $\mP$ is entangled. We also have $\mP^k=\frac{1}{4}\left(e e^{\top}\right) \otimes\left(e e^{\top}\right)$ for $k \geq 2$. Then $(I-\gamma P)^{-1}=\mI+\frac{\gamma+\gamma^2}{4}\left(e e^{\top}\right) \otimes\left(e e^{\top}\right)+\delta \gamma\left(\epsilon e^{\top}\right) \otimes\left(e \epsilon^{\top}\right)$. Then for any $\vr_A, \vr_B$, we have \[(\mI-\gamma \mP)^{-1}\left(\vr_A \otimes e+e \otimes \vr_B\right)=\vr_A \otimes e+h_A\left(\gamma+\gamma^2\right) / 2 e \otimes e+\vr_B \otimes e+h_B\left(\gamma+\gamma^2\right) / 2 e \otimes e+2 \delta \gamma\left(\epsilon^{\top} \vr_B\right) \epsilon \otimes e\,,\]
where $h_A=e^{\top} \vr_A, h_B=e^{\top} \vr_B$.
\section{Proof of Theorem~\ref{thm: two-agent atv}}\label{app: Proof of atv}
Let $\mP_A, \mP_B$ be the optimal solution to~\eqref{eq: degree of independent} w.r.t agent $A,B$. For any subset of state-action pairs of agent $A$, $\gF\subseteq \gS_A\times\gA_A$, we have
\begin{align*}
    &\abs{\sum_{s_A^\prime,a_A^\prime\in \gF}\pa{\mP^\pi_A-\mP_A}_{(s_A^\prime,a_A^\prime\mid s_A,a_A)}}\\
=&\abs{\sum_{s_A^\prime,a_A^\prime\in \gF}\sum_{s_B^\prime,a_B^\prime}\sum_{s_B,a_B} \pa{\mP^\pi_{AB}-\mP_A\otimes \mP_B}_{(\vs^\prime,\va^\prime\mid \vs,\va)} \mu^\pi_{AB}(s_B,a_B\mid s_A,a_A)}\\
\leq&\sum_{s_B,a_B} \abs{ \sum_{s_A^\prime,a_A^\prime\in \gF}\sum_{s_B^\prime,a_B^\prime}\pa{\mP^\pi_{AB}-\mP_A\otimes \mP_B}_{(\vs^\prime,\va^\prime\mid \vs,\va)} } \mu^\pi_{AB}(s_B,a_B\mid s_A,a_A)\\
\leq&  \sum_{s_B,a_B} E_A(\PAB) \mu^\pi_{AB}(s_B,a_B\mid s_A,a_A) = E_A(\PAB)
\end{align*}
where the last inequality follows from the definition of agent-wise total variation distance. Since the result holds for any $\gF$ and $(s_A,a_A)\in\gS_A\times\gA_A$, we have  
\[\norm{\mP^\pi_A-\mP_A }_{\rm{TV}}\leq E_A(\PAB)\,,\]and similar results hold for $\mP^\pi_B$.

Next we have 

    \begin{align*}
        &\pa{\mI-\gamma \mP^\pi_{AB}}^{-1} (\vr_A\otimes \ve) - \pa{\pa{\mI-\gamma \mP^\pi_A}^{-1}\vr_A} \otimes \ve\\
        = & \pa{\mI-\gamma \mP^\pi_{AB}}^{-1} (\vr_A\otimes \ve)- \pa{\mI-\gamma\mP_A\otimes\mP_B  }^{-1} (\vr_A\otimes \ve) \\
        &\qquad +\pa{\mI-\gamma\mP_A\otimes\mP_B  }^{-1} (\vr_A\otimes \ve) - \pa{\pa{\mI-\gamma \mP^\pi_A}^{-1}\vr_A} \otimes \ve\\
        \stackrel{(i)}{=} & \underbrace{(\mI-\gamma \mP^\pi_{AB})^{-1} (\vr_A\otimes \ve)- \pa{\mI-\gamma\mP_A\otimes\mP_B  }^{-1} (\vr_A\otimes \ve)}_{(I)} \\
        &\qquad +\underbrace{\pa{\pa{\mI-\gamma\mP_A }^{-1} \vr_A}\otimes \ve - \pa{\pa{\mI-\gamma \mP^\pi_A}^{-1}\vr_A} \otimes \ve}_{(II)}\\
    \end{align*}
    where $(i)$ also follows the same ``absorbing'' technique in the proof of Theorem~\ref{thm: mixed_state}.

    For $(I)$, apply Lemma~\ref{lem: matrix inverse}, it holds
    \begin{align*}
        &\norm{(\mI-\gamma \mP^\pi_{AB})^{-1} (\vr_A\otimes \ve)- \pa{\mI-\gamma\mP_A\otimes\mP_B  }^{-1} (\vr_A\otimes \ve)}_\infty\\
        =&\norm{(\mI-\gamma \mP^\pi_{AB})^{-1}\pa{\gamma\mP^\pi_{AB}-\gamma\mP_A\otimes\mP_B} \pa{\mI-\gamma\mP_A\otimes\mP_B  }^{-1} (\vr_A\otimes \ve) }_\infty\\
        \leq & \norm{(\mI-\gamma \mP^\pi_{AB})^{-1}}_\infty\norm{\pa{\gamma\mP^\pi_{AB}-\gamma\mP_A\otimes\mP_B} \pa{\pa{\mI-\gamma\mP_A }^{-1} \vr_A}\otimes \ve }_\infty\\
        \overset{(i)}{\leq} & \norm{(\mI-\gamma \mP^\pi_{AB})^{-1}}_\infty 2\gamma E_A(\PAB) \norm{\pa{\mI-\gamma\mP_A }^{-1} \vr_A}_\infty\\
        \leq & \frac{2\gamma E_A(\PAB)r_{\max}^A}{1-\gamma}\norm{(\mI-\gamma \mP^\pi_{AB})^{-1}}_\infty
    \leq  \frac{2\gamma E_A(\PAB) r_{\max}^A}{(1-\gamma)^2}\,,
    \end{align*}
    where $(i)$ follows by the definition of agent-wise total variation distance when $\|\vr_A\|_\infty\ne0$, and also trivially hold when $\|\vr_A\|_\infty=0$. Similarly, for $(II)$ we have
    \begin{align*}
        &\norm{\pa{\pa{\mI-\gamma\mP_A }^{-1} \vr_A}\otimes \ve - \pa{\pa{\mI-\gamma \mP^\pi_A}^{-1}\vr_A} \otimes \ve}_\infty\\
        =&\norm{\pa{\pa{\mI-\gamma\mP_A}^{-1} -\pa{\mI-\gamma \mP^\pi_A}^{-1}}\vr_A }_\infty\\
        =&\norm{\pa{\mI-\gamma \mP^\pi_A}^{-1}\pa{\gamma\mP^\pi_A-\gamma\mP_A}\pa{\mI-\gamma\mP_A}^{-1}\vr_A }_\infty\\
        \leq & \frac{2\gamma E_A(\PAB) r_{\max}^A}{(1-\gamma)^2}\,.
    \end{align*}
    Then we have 
    \begin{align*}
        \norm{\pa{\mI-\gamma \mP^\pi_{AB}}^{-1} (\vr_A\otimes \ve) - \pa{\pa{\mI-\gamma \mP^\pi_A}^{-1}\vr_A} \otimes \ve}_\infty \leq \frac{4\gamma E_A(\PAB) r_{\max}^A}{(1-\gamma)^2}\,. 
    \end{align*}
    We can derive similar results for agent $B$, i.e.,
    \begin{align*}
        \norm{\pa{\mI-\gamma \mP^\pi_{AB}}^{-1} (\ve\otimes\vr_B) - \ve\otimes\pa{\pa{\mI-\gamma \mP^\pi_B}^{-1}\vr_B} }_\infty \leq \frac{4\gamma E_B(\PAB) r_{\max}^B}{(1-\gamma)^2}\,. 
    \end{align*}
    Put it all together we have 
    \[\Big\lVert Q^\pi_{AB}-\pa{Q^\pi_A\otimes \ve + \ve \otimes Q^\pi_B} \Big\rVert_\infty \leq \frac{4\gamma (E_A(\PAB)r_{\max}^A+E_B(\PAB)r^B_{\max})}{(1-\gamma)^2}\,.\]
Finally, the proof of Theorem~\ref{thm: two-agent value decomp} follows as an immediate corollary of Theorem~\ref{thm: two-agent atv}.
\section{Proof of Theorem~\ref{thm: rho-weighted decomp}}\label{app: proof of mu-weight}
We provide the proof for two agents here, one can easily generalize the proof to multi-agent scenarios. Compared to the proof of Theorem~\ref{thm: two-agent atv}, this proof follows similar framework and differs in several details. 

The first one is the following lemma for the ``localized'' stationary distribution
\begin{lemma}\label{lem: marginal_stationary_dis}
    $\mP^\pi_{A}$ has stationary distribution $\mu^\pi_A$ with \[\forall (s_A,a_A)\,,\,\mu^\pi_A(s_A,a_A)=\sum_{s_B,a_B}\mu^\pi_{AB}(s_A, s_B,a_A,a_B)\,.\]
\end{lemma}
In other words, the local stationary distribution of each agent is exactly the marginal distribution of global $\mu^\pi_{AB}$.
\begin{proof}[Proof of Lemma~\ref{lem: marginal_stationary_dis}]
    We proof by verify the definition of stationary distribution. For any $(s_A^\prime,a_A^\prime)$, it holds
    \begin{align*}
        &\sum_{s_A,a_A} \pa{\sum_{s_B,a_B} \mu^\pi_{AB}(s_A, s_B,a_A,a_B)} P^\pi(s_A^\prime,a_A^\prime\mid s_A,a_A)\\
        =&\sum_{s_A,a_A} \sum_{s_B,a_B} \mu^\pi_{AB}(s_A, s_B,a_A,a_B)\sum_{s_B^\prime,a_B^\prime} \sum_{s_B^{\prime\prime},a_B^{\prime\prime}}P^\pi\pa{s_A^{\prime},s_B^\prime,a_A^\prime,a_B^\prime\mid s_A,s_B^{\prime\prime},a_A,a_B^{\prime\prime}} \mu^\pi_{AB}(s_B^{\prime\prime},a_B^{\prime\prime}\mid s_A,a_A)\\
        =&\sum_{s_A,a_A} \sum_{s_B,a_B} \mu^\pi_{AB}(s_B,a_B\mid s_A,a_A)\sum_{s_B^\prime,a_B^\prime} \sum_{s_B^{\prime\prime},a_B^{\prime\prime}}P^\pi\pa{s_A^\prime,s_B^\prime,a_A^\prime,a_B^\prime\mid s_A,s_B^{\prime\prime},a_A,a_B^{\prime\prime}} \mu^\pi_{AB}(s_A, s_B^{\prime\prime},a_A,a_B^{\prime\prime})\\
        =& \sum_{s_A,a_A} \sum_{s_B^\prime,a_B^\prime} \sum_{s_B^{\prime\prime},a_B^{\prime\prime}}P^\pi\pa{s_A^\prime,s_B^\prime,a_A^\prime,a_B^\prime\mid s_A,s_B^{\prime\prime},a_A,a_B^{\prime\prime}} \mu^\pi_{AB}(s_A, s_B^{\prime\prime},a_A,a_B^{\prime\prime})\\
        =&\sum_{s_B^\prime,a_B^\prime}\mu^\pi_{AB}(s_A^\prime, s_B^\prime,a_A^\prime,a_B^\prime)\,.
    \end{align*}
    where the last equation follows from the definition of $\mu^\pi_{AB}$. Hence we conclude that $\sum_{s_B,a_B}\mu^\pi_{AB}(s_A, s_B,a_A,a_B)$ is a stationary distribution of $\mP^\pi_A$.
\end{proof}
We are then ready to prove Theorem~\ref{thm: rho-weighted decomp}. We first note that similar to ATV distance in~\eqref{eq: degree of independent}, the optimal solution to $E_A(\PAB)$ w.r.t $\muAB$-weighted ATV distance also only depends on $\mP_A$. Thus, let $\mP_A,\mP_B$ be the optimal solutions to $E_A(\PAB),E_B(\PAB)$ respectively. 

Let $\vx\in\R^{|\gS_A||\gA_A|}$ with $\|\vx\|_\infty=1$. Following the same technique in the proof of Theorem~\ref{thm: rho-weighted decomp}, we have
\begin{align*}
    &\mu^{\pi^\top}_A\abs{\pa{\mP^\pi_{A}-\mP_A}\vx}\\
    =&\sum_{s_A,a_A}\mu^\pi_A(s_A,a_A)\abs{\sum_{s_A^\prime,a_A^\prime}\pa{\mP^\pi_A-\mP_A}_{(s_A^\prime,a_A^\prime\mid s_A,a_A)}\vx(s_A^\prime,a_A^\prime)}\\
=&\sum_{s_A,a_A}\mu^\pi_A(s_A,a_A)\abs{\sum_{s_A^\prime,a_A^\prime}\vx(s_A^\prime,a_A^\prime)\sum_{s_B^\prime,a_B^\prime}\sum_{s_B,a_B} \pa{\mP^\pi_{AB}-\mP_A\otimes\mP_B}_{(\vs^\prime,\va^\prime\mid \vs,\va)} \mu^\pi_{AB}(s_B,a_B\mid s_A,a_A)}\\
\leq&\sum_{\vs,\va}  \abs{\sum_{s_A^\prime,a_A^\prime}\vx(s_A^\prime,a_A^\prime)\sum_{s_B^\prime,a_B^\prime}\pa{\mP^\pi_{AB}-\mP_A\otimes\mP_B}_{(\vs^\prime,\va^\prime\mid \vs,\va)} } \mu^\pi_{AB}(\vs,\va)\leq 2 E_A(\PAB)
\end{align*}
where the second last inequality follows from Lemma~\ref{lem: marginal_stationary_dis}. We then conclude  
\[\norm{\mP^\pi_A-\mP_A }_{\mu,\infty}\leq 2E_A(\PAB)\,,\]and similar results hold for $\mP^\pi_B$. We then apply the decomposition

\begin{align*}
        &\pa{\mI-\gamma \mP^\pi_{AB}}^{-1} (\vr_A\otimes \ve) - \pa{\pa{\mI-\gamma \mP^\pi_A}^{-1}\vr_A} \otimes \ve\\
        = & \underbrace{(\mI-\gamma \mP^\pi_{AB})^{-1} (\vr_A\otimes \ve)- \pa{\mI-\gamma\mP_A\otimes\mP_B  }^{-1} (\vr_A\otimes \ve)}_{(I)} \\
        &\qquad +\underbrace{\pa{\pa{\mI-\gamma\mP_A  }^{-1} \vr_A}\otimes \ve - \pa{\pa{\mI-\gamma \mP^\pi_A}^{-1}\vr_A} \otimes \ve}_{(II)}\\
\end{align*}
For $(I)$, we have
\begin{align*}
    &\norm{(\mI-\gamma \mP^\pi_{AB})^{-1} (\vr_A\otimes \ve)- \pa{\mI-\gamma\mP_A\otimes\mP_B  }^{-1} (\vr_A\otimes \ve)}_{\muAB}\\
        =&\norm{(\mI-\gamma \mP^\pi_{AB})^{-1}\pa{\gamma\mP^\pi_{AB}-\gamma\mP_A\otimes\mP_B} \pa{\mI-\gamma\mP_A\otimes\mP_B  }^{-1} (\vr_A\otimes \ve) }_{\muAB}\\
        \overset{(i)}{\leq}&\frac{1}{1-\gamma}\norm{\pa{\pa{\gamma\mP^\pi_{AB}-\gamma\mP_A\otimes\mP_B} \pa{\mI-\gamma\mP_A  }^{-1}\vr_A }\otimes \ve }_{\muAB}\\
        \leq& \frac{2\gamma E(\pi)}{1-\gamma}\norm{\pa{\mI-\gamma\mP_A  }^{-1} \vr_A}_\infty\leq \frac{2\gamma E(\pi) r_{\max}}{(1-\gamma)^2}\,,
\end{align*}
where $(i)$ follows from the fact that for any $\vx$
\[\|\mP\vx\|_\mu=\mu^\top|\mP\vx|\leq \mu^\top\mP|\vx|=\mu^\top|\vx|=\|\vx\|_\mu\,.\]

For $(II)$ one can use Lemma~\ref{lem: marginal_stationary_dis} to verify 
\begin{align*}
    &\norm{\pa{\pa{\mI-\gamma\mP_A  }^{-1} \vr_A}\otimes \ve - \pa{\pa{\mI-\gamma \mP^\pi_A}^{-1}\vr_A} \otimes \ve}_{\muAB}\\
    =&\norm{\pa{\mI-\gamma\mP_A  }^{-1} \vr_A - \pa{\mI-\gamma \mP^\pi_A}^{-1}\vr_A }_{\muA}
\end{align*}
And similar results to $(I)$ holds. We then conclude the proof of Theorem~\ref{thm: rho-weighted decomp}.

\section{Results for Multi-agent MDPs}\label{app: multi-agent}
For clarity, we use superscript $s^i$ to denote the $i$-th element in state space and subscript $s_i$ to represent the state at $i$-th arm. Furthermore, we denote $\gS^{-i}\coloneqq \gS\setminus s^i$ and $\vs\coloneqq s_{1:N}\coloneqq\{s_1,s_2,\ldots,s_N\}$ is the profile of $N$-arms.

Given any global policy $\pi$, for any agent $i\in[N]$,
\begin{equation*}
    P^\pi_i(s_i^\prime, a_i^\prime\mid s_i,a_i)= \sum_{s_{-i}^\prime,a_{-i}^\prime} \sum_{s_{-i},a_{-i}}P^\pi_{1:N}\pa{s_{1:N}^\prime,a_{1:N}^\prime\mid s_{1:N},a_{1:N}} \rho_{1:N}^\pi(s_{-i},a_{-i}\mid s_i,a_i) \,.
\end{equation*}

\begin{definition}[Measure of Multi-agent Markov Entanglement]
 Consider a $N$-agent Markov system $\gM_{1:N}$ with joint state space $\gS=\times_{i=1}^N\gS_i$ and action space $\gA=\times_{i=1}^N\gA_i$. Given any policy $\pi\colon \gS \to \Delta(\gA)$, the measure of Markov entanglement of $N$ agents is
\begin{equation*}
    E(\PN) = \min_{\mP\in\SEP } d(\mP^\pi_{1:N}, \mP)\,,
\end{equation*}
where $d(\cdot,\cdot)$ is some distance measure.
\end{definition}

The following theorem generalizes the results of value-decomposition for two-agent Markov systems in Theorem~\ref{thm: two-agent atv} to multi-agent Markov systems. 
\begin{theorem}\label{thm: multi-agent value decomp}
    Consider a $N$-agent MDP $\gM_{1:N}$ and policy $\pi\colon \gS \to \Delta(\gA)$ with the measure of Markov entanglement $E_i(\PN)$ w.r.t ATV distance, it holds for any agent $i$,
    \[\norm{\mP^\pi_i-\mP_i }_{\infty}\leq 2_iE(\PN)\,.\]
    where $\mP_i$ is the optimal solution of \eqref{eq: Multi-agent Markov Entanglement}. Furthermore, the decomposition error is entry-wise bounded by the measure of Markov entanglement,
    \[\norm{ Q^\pi_{1:N}(\vs,\va)-\sum_{i=1}^N Q^\pi_i(s_i,a_i) }_\infty \leq \frac{4\gamma \pa{\sum_{i=1}^NE_i(\PN)r_{\max}^i}}{(1-\gamma)^2}\,.\]
\end{theorem}
The proof mainly follows the following lemma, which generalizes the key technique used in Theorem~\ref{thm: mixed_state}. 
\begin{lemma}
For any agent $i$, it holds
    \begin{equation}
        \pa{\sum_{j=1}^K x_j\mP_1^{(j)}\otimes \mP_2^{(j)}\otimes\cdots \otimes\mP_N^{(j)} }\cdot\pa{(\ve\otimes)^{i-1}\vr_i(\otimes \ve)^{N-i} } = (\ve\otimes)^{i-1}\pa{\sum_{j=1}^K x_j \mP_i^{(j)} \vr_i}(\otimes \ve)^{N-i}\,.
    \end{equation}
\end{lemma}
The lemma follows from the property of tensor product. 
\section{Proof of Theorem~\ref{thm: RMAB}}\label{app: index}

One caveat here is that we have to restrict chaotic behaviors in the mean-field limit. We thus introduce two technical assumptions.

We first define the transition of configuration under index policy $\pi$ as $\phi^\pi\colon\Delta^{|\gS|}\to\Delta^{|\gS|}$ such that
\[\phi^\pi(\vm)=\E\br{\vm[t+1]\mid \vm[t]=\vm, \pi}\,.\]
For $t>0$, we denote $\Phi_t\coloneqq (\phi^\pi)^t$ apply the transition mapping for $t$ rounds. 
\begin{assumption}[Uniform Global Attractor Property (UGAP)] 
    There exists a uniform global attractor $\vm^\ast$ of $\phi^\pi(\cdot)$, i.e. for all $\varepsilon > 0$, there exists $T(\varepsilon)$ such
that for all $t \geq T (\varepsilon)$ and all $\vm\in\Delta^{|\gS|}$, one has $\norm{\Phi_t(\vm)-\vm^\ast}_\infty<\varepsilon$.
\end{assumption}

The UGAP assumption ensures the uniqueness of $\vm^\ast$ and guarantees fast convergence from any initial $\vm$ to $\vm^\ast$.

\begin{assumption}[Non-degenerate RMAB]
There exists state $s\in\gS$ such that $0<\pi^\ast(s,0)<1$, where $\pi^\ast$ is the policy under $\vm^\ast$. 
\end{assumption}

The non-degenerate assumption further restricts cyclic behavior in the mean-field limit. 

Non-degenerate and UGAP are two standard technical assumptions for the index policy, which restrict chaotic behavior in asymptotic regime and will be further introduced in subsequent sections. We note here these two assumptions are also used in almost all theoretical work on index policies (\citealt{weber90on,Verloop2016AsymptoticallyOP,gast23exponential,gast24linear}).

\emph{Proof of Theorem~\ref{thm: RMAB}.} In the subsequent proof, we let $\nu_1>\nu_2>\nu_3>\cdots>\nu_{|S|}$. This does not lose generality in that we can always exchange state index. The proof consists of several steps
\paragraph{Step 1: Find $\vm^\ast$}
Recall the transition mapping for configurations $\phi^\pi\colon\Delta^{|\gS|}\to\Delta^{|\gS|}$,
\[\phi^\pi(\vm)=\E\br{\vm[t+1]\mid \vm[t]=\vm, \pi}\,.\]
Notice that the definition of $\phi^\pi$ does not depend on $N$. We adapt from Lemma B.1 in \cite{gast23exponential} defined specially for Whittle Index, 
\begin{lemma}[Piecewise Affine]\label{lem: piecewise affine}
    Given any index policy $\pi$, $\phi^\pi$ is a piecewise affine continuous function with $|\gS|$ affine pieces.
\end{lemma}
When the context is clear, we abbreviate $\phi^\pi$ as $\phi$. For any $\vm\in\Delta^{|\gS|}$, define $s(\vm)\in[|\gS|]$ be the state such that $\sum_{i=1}^{s(\vm)-1}\vm_i\leq \alpha<\sum_{i=1}^{s(\vm)}m_i$. Lemma \ref{lem: piecewise affine} characterizes for any $\vm\in\gZ_i\coloneqq\brk{\vm\in\Delta^{|\gS|}\mid s(\vm)=i}$, there exists $\mK_{s(\vm)}, \vb_{s(\vm)}$ such that \[\phi(\vm)=\mK_{s(\vm)}\vm+\vb_{s(\vm)}\,.\]
By Brouwer fixed point theorem, there exists a fixed point $\vm^\ast$ such that $\phi(\vm^\ast)=\vm^\ast$. The UGAP condition guarantees the uniqueness of $\vm^\ast$. Our choice of $\pi^\ast$ is the corresponding policy under $\vm^\ast$.

\paragraph{Step 2: Connecting policy entanglement with the deviation of stationary distribution} 
Combine Proposition~\ref{prop: entanglement of policy} with the RMAB model, we have 
\begin{lemma}\label{lem: policy_entangle with deviation}
The measure of Markov entanglement w.r.t $\mu^\pi_{1:N}$-weighted ATV distance is bounded by the deviation of mean-field configuration,
    \[E_i(\pi)\leq |\gS|^2\cdot\E\br{\|\vm-\vm^\ast\|_\infty}\,,\]
    where the expectation is taking over the stationary distribution $\vm\sim \mu^\pi_{1:N}$.
\end{lemma}

\begin{proof}
Given the homogeneity of agents, we first demonstrate for any two agent $i,j$, it holds
\[\sum_{s_{1:N}} \mu^\pi(s_{1:N})\abs{\pi(a_i=a\mid s_{1:N})-\pi^\ast(a_i=a\mid s_i)}=\sum_{s_{1:N}} \mu^\pi(s_{1:N})\abs{\pi(a_j=a\mid s_{1:N})-\pi^\ast(a_j=a\mid s_i)}\,.\]
To see this, we first notice by the definition of index policy 
\[\abs{\pi(a_i=a\mid s_i=s, \vm)-\pi^\ast(a\mid s)}=\abs{\pi(a_j=a\mid s_{j}=s,\vm)-\pi^\ast(a\mid s)}\,.\]
It then suffices to prove $\sum_{s_i=s, s_{1:N}=\vm}\mu(s_{1:N})= \sum_{s_j=s, s_{1:N}=\vm}\mu(s_{1:N})$. If $\sum_{s_i=s, s_{1:N}=\vm}\mu(s_{1:N})\leq \sum_{s_j=s, s_{1:N}=\vm}\mu(s_{1:N})$, we can exchange the agent index of $i$ and $j$. This will result in the same stationary distribution and $\sum_{s_i=s, s_{1:N}=\vm}\mu(s_{1:N})\geq \sum_{s_j=s, s_{1:N}=\vm}\mu(s_{1:N})$ and thus the equation. 
We then rewrite the bound in Proposition~\ref{prop: entanglement of policy}, 
\begin{align*}
E(\pi)&\leq \frac{1}{2}\sup_i\sum_{s_{1:N}} \mu^\pi(s_{1:N})\sum_{a_i}\abs{\pi(a_i\mid s_{1:N})-\pi^\ast(a_i\mid s_i)}\\
&= \sup_i\sum_{s_{1:N}} \mu^\pi(s_{1:N})\abs{\pi(a_i=1\mid s_{1:N})-\pi^\ast(a_i=1\mid s_i)}\\
&= \frac{1}{N} \sum_{s_{1:N}} \mu^\pi(s_{1:N})\sum_{i=1}^N\abs{\pi(a_i=1\mid s_{1:N})-\pi^\ast(a_i=1\mid s_i)}\\
&=\sum_\vm \mu^\pi(\vm)\sum_{s\in\gS} \vm_s \abs{\pi(a=1\mid s,\vm)-\pi^\ast(a=1\mid s)}
\end{align*}
For any configuration $\vm$ and state $s$, we have 
\begin{align*}
    &\vm_s\abs{\pi(a=1\mid s,\vm)-\pi^\ast(a=1\mid s)}\\
    =& \vm_s\abs{\frac{\pi^\ast(a=1\mid s)\vm_s^\ast N+k_s}{\vm_s^\ast N+\ell_s}-\pi^\ast(a=1\mid s)}\\
    =&\frac{\vm_s^\ast N+\ell_s}{N} \abs{\frac{k_s-\ell_s\pi^\ast(a=1\mid s)}{\vm_s^\ast N+\ell_s}}\\
    \leq& |\gS|\|\vm-\vm^\ast\|_\infty\,,
\end{align*}
where $\abs{k_s}\leq (|\gS|-1) \|\vm-\vm^\ast\|_\infty N$ representing the additional fraction of state $s$ to be activated due to the deviation from $m^\ast$ and $\abs{\ell_s}\leq \|\vm-\vm^\ast\|_\infty N$ representing the deviation of $\vm_s$ from $\vm_s^\ast$. The results then hold by taking summation over $s$ and expectation over $\vm$.

\end{proof}

\paragraph{Step 3: Concentrations and local stability} 
To bound $\E\br{\norm{\vm-\vm^\ast}_\infty}$, we start with several technical lemmas from previous RMAB literature. We use the same notation $\Phi_t=\phi(\Phi_{t-1})$. 
\begin{lemma}[One-step Concentration, Lemma 1 in \cite{gast24linear}]
    Let $\eps[1]=\vm[1]-\phi(\vm[0])$, it holds\[\E\br{\|\eps[1]\|_1\mid \vm[0]}\leq \sqrt\frac{|\gS|}{N}\,.\] 
\end{lemma}
\begin{lemma}[Multi-step Concentration, Lemma C.4 in \cite{gast23exponential}]
    There exists a positive constant $K$ such that for all $t\in \mathbb{N}$ and $\delta >0$,
    \[\Pr\br{\norm{\vm[t]-\Phi_{t}(\vm)}_\infty\geq(1+K+K^2+\cdots+K^t)\delta\mid \vm[0]=\vm}\leq t|\gS|e^{-2N\delta^2}\]
\end{lemma}

\begin{lemma}[Local Stability, Lemma C.5 in \cite{gast23exponential}]\label{lem: Lemma C}
    Under non-degenerate and UGAP:
    \begin{itemize}
        \item [(i)] $\mK_{s(\vm^\ast)}$ is a stable matrix, i.e. its spectral radius is strictly less than $1$.
        \item [(ii)] For any $\eps$, there exists $T(\eps)>0$ such that for all $\vm\in\Delta^{|\gS|}$, $\norm{\Phi_{T(\eps)}(\vm)-\vm^\ast}_\infty< \eps$.
    \end{itemize}
\end{lemma}
The first result implies there exists some matrix norm $\norm{\cdot}_\beta$ such that $\norm{\mK_{s(\vm^\ast)}}_\beta<1$. By the equivalence of norms, there exists constant $C^1_\beta,C^2_\beta>0$ such that for all $\vx\in\R^{|\gS|}$ 
\[C^1_\beta\|\vx\|_\beta\leq \norm{\vx}_\infty\leq C^2_\beta\|\vx\|_\beta\,.\]
Combine the second result of Lemma~\ref{lem: Lemma C} and non-degenerate condition, we can construct a neighborhood $\gN$ of $\vm^\ast$ such that $\gN=\gB(\vm^\ast,\eps)\cap \Delta^{|\gS|}\in\gZ_{s(\vm^\ast)}$ where $\eps>0$ and $\gB(\vm^\ast,\eps)=\brk{\vm\mid \norm{\vm-\vm^\ast}_\infty<\eps}$ is an open ball. We next show that $\vm[0]$ under stationary distribution will concentrate in $\gN$ with high probability. Let $\Tilde{T}=T(\eps/2)$ such that for all $\vm\in\Delta^{|\gS|}$, $\norm{\Phi_{\Tilde{T}}(\vm)-\vm^\ast}_\infty< \eps/2$. It holds 
\begin{align*}
    \Pr\br{\vm[0]\ne \gN}&=\Pr\br{\norm{\vm[0]-\vm^\ast}_\infty\geq \eps}\\
    &\overset{(i)}{=}\Pr\br{\norm{\vm[\Tilde{T}]-\vm^\ast}_\infty\geq \eps\mid \vm[0]=\vm}\\
    &\leq \Pr\br{\norm{\vm[\Tilde{T}]-\Phi_{\Tilde{T}}(\vm)}_\infty\geq \frac{\eps}{2}\mid \vm[0]=\vm}+\Pr\br{\norm{\Phi_{\Tilde{T}}(\vm)-\vm^\ast}_\infty\geq \frac{\eps}{2}}\\
    &=\Pr\br{\norm{\vm[\Tilde{T}]-\Phi_{\Tilde{T}}(\vm)}_\infty\geq \frac{\eps}{2}\mid \vm[0]=\vm}\leq \Tilde{T}|\gS|e^{-2uN}
\end{align*}
where $(i)$ follows from the stationarity $\vm[\Tilde{T}]$ and $\vm[0]$ are \emph{i.i.d} and the constant $u=\pa{\frac{\eps}{2(1+K+K^2+\cdots+K^{\Tilde{T}})}}^2$ does not depend on $N$.

\paragraph{Step 4: Put it together}
Finally, we are ready to bound $\E\br{\norm{\vm-\vm^\ast}_\infty}$. Notice for all $\vm[0]\in \gN$, we have 
\begin{align*}
    \vm[1]-\vm^\ast &= \phi(\vm[0])+\eps[1]-\vm^\ast\\
    &=\mK_{s(\vm^\ast)}\pa{\vm[0]-\vm^\ast} + \eps[1]\,.
\end{align*}
Taking $\norm{\cdot}_\beta$ on both side, 
\begin{align*}
    \norm{\vm[1]-\vm^\ast}_\beta &\leq\norm{\mK_{s(\vm^\ast)}\pa{\vm[0]-\vm^\ast}}_\beta + \norm{\eps[1]}_\beta\\
    &\leq \norm{\mK_{s(\vm^\ast)}}_\beta \norm{\vm[0]-\vm^\ast}_\beta+ \norm{\eps[1]}_\beta\,.
\end{align*}
Taking expectation on both side, 
\begin{align*}
    &\E\br{\norm{\vm[1]-\vm^\ast}_\beta} \\
    =&\E\br{\norm{\phi(\vm[0])-\vm^\ast}_\beta\cdot\mathbf{1}\brk{\vm[0]\in\gN}}+\E\br{\norm{\phi(\vm[0])-\vm^\ast}_\beta\cdot\mathbf{1}\brk{\vm[0]\notin\gN}}+\E\br{\norm{\eps[1]}_\beta}\\
    \leq& \norm{\mK_{s(\vm^\ast)}}_\beta \E\br{\norm{\vm[0]-\vm^\ast}_\beta\cdot\mathbf{1}\brk{\vm[0]\in\gN}}+ \Pr\br{\vm[0]\notin\gN} \sup_{\vm[0]}\norm{\phi(\vm[0])-\vm^\ast}_\beta +\E\br{\norm{\eps[1]}_\beta}\\
    \leq& \norm{\mK_{s(\vm^\ast)}}_\beta \E\br{\norm{\vm[0]-\vm^\ast}_\beta}+ \Pr\br{\vm[0]\notin\gN} \sup_{\vm[0]}\norm{\phi(\vm[0])-\vm^\ast}_\beta +\E\br{\norm{\eps[1]}_\beta}\,.
\end{align*}
By stationarity, one have $\E\br{\norm{\vm[1]-\vm^\ast}_\beta}=\E\br{\norm{\vm[0]-\vm^\ast}_\beta}$. This refines the above inequality, 
\begin{align*}
    \E\br{\norm{\vm[0]-\vm^\ast}_\infty}&\leq \frac{C^2_\beta}{1-\norm{\mK_{s(\vm^\ast)}}_\beta} \pa{\sup_{\vm[0]}\Pr\br{\vm[0]\notin\gN} \norm{\phi(\vm[0])-\vm^\ast}_\beta +\E\br{\norm{\eps[1]}_\beta}}\\
    &\leq \frac{C^2_\beta}{C^1_\beta(1-\norm{\mK_{s(\vm^\ast)}}_\beta)} \pa{\Pr\br{\vm[0]\notin\gN} + \E\br{\norm{\eps[1]}_\infty}}\\
    &\leq \frac{C^2_\beta}{C^1_\beta(1-\norm{\mK_{s(\vm^\ast)}}_\beta)} \pa{\Tilde{T}|\gS|e^{-2uN} + \frac{\sqrt{|\gS|}}{\sqrt{N}}}\,.
\end{align*}
We combine Lemma~\ref{lem: policy_entangle with deviation} and conclude the proof of Theorem~\ref{thm: RMAB}.
\section{Extensions of Markov entanglement}\label{app: extension}
We explore several extensions of Markov entanglement theory to other structured multi-agent MDPs.
\subsection{Coupled MDPs with Exogenous Information}\label{app: exog_order}

In many practical scenarios, the agents' transitions and actions are coupled by a shared exogenous signal. For example, in ride-hailing platforms, the specific dispatch is related to the exogenous order at the current moment (\citealt{qin20ride, han22real, xabi24better}); in warehouse routing, the scheduling of robots is also related to the exogenous task revealed so far (\citealt{chan2024the}). 

We will then enrich our framework by incorporating these exogenous information. At each timestep $t$, there will an exogenous information $z_t$ revealed to the decision maker. $z_t$ is assumed to evolve following a Markov chain independent of the action and transition of agents. We assume $z_t\in\gZ$ and $\gZ$ is finite.

Given the current state $\vs$ and exogenous information $z$, the policy is given by $\pi:\gS\times\gZ\to\Delta(\tilde\gA)$, where $\tilde\gA$ refers to the set of feasible actions. We then have the global transition depending on exogenous information $z$,
\[P^\pi_{ABz}(\vs^\prime,\va^\prime,z^\prime\mid \vs,\va,z)=P(\vs^\prime\mid \vs,\va,z)\cdot \pi(\va^\prime\mid\vs^\prime,z^\prime)\cdot P(z^\prime\mid z)\,.\]
and global Q-value $Q^\pi_{ABz}\in\R^{|\gS|^N|\gA|^N|\gZ|}$, 
\[Q^\pi_{AB}(\vs,\va,z)=\E\br{\sum_{t=0}^\infty \sum_{i=1}^Nr(s_{i,t},a_{i,t},z_t)\mid \vs_0=\vs,\va_0=\va,z_0=z}\,.\]
We assume the system is unichain and the stationary distribution is $\mu^\pi_{ABz}$. Then we can derive the local transition under new algorithm by 
\[P_{Az}(s_A^\prime,a_A^\prime,z^\prime\mid s_A,a_A,z)=\sum_{s_B,a_B}\mu_{ABz}^\pi(s_B,a_B\mid s_A,a_A,z)\sum_{s_B^\prime,a_B^\prime}P^\pi_{ABz}(\vs^\prime,\va^\prime,z^\prime\mid \vs,\va,z)\,,\]
Given the local transition, we have the local value $\mQ_{Az}^\pi=(\mI-\gamma\mP_{Az})^{-1}(\vr_{Az})$ via Bellman Equation. 

Combined with exogenous information, we consider the following value decomposition
\[Q^\pi_{AB}(\vs,\va,z)= Q^\pi_A(s_A,a_A,z) + Q^\pi_B(s_B,a_B,z)\,.\]
We start by introducing agent-wise Markov entanglement defined for each agent
\begin{equation}
\mP^\pi_{ABz}=\sum_{j=1}^Kx_j\mP_{Az}^{(j)}\otimes \mP_B^{(j)}\,.    
\end{equation}

\begin{proposition}
If the system is agent-wise separable for all agents, then 
\[\mQ_{ABz}^\pi = \mQ_{Az}^\pi \otimes \ve_{|\gS||\gA|} + \ve_{|\gS||\gA|}\otimes \mQ_{Bz}^\pi\,.\]
\end{proposition}
\begin{proof}
The proof is basically the same as Theorem~\ref{thm: mixed_state}. One can first quickly show that $\mP_{Az}=\sum_{j=1}^Kx_j\mP_{Az}^{(j)}$. And then it holds
\begin{align*}
        &\pa{\sum_{j=1}^K x_j\mP_{Az}^{(j)}\otimes \mP_B^{(j)}}^t \pa{\vr_A\otimes \ve_{|z|} \otimes \ve_{|\gS||\gA|} } \\
        =& \pa{\sum_{j=1}^K x_j\mP_{Az}^{(j)}\otimes \mP_B^{(j)}}^{t-1}\pa{\sum_{j=1}^K x_j \pa{ \mP_{Az}^{(j)}(\vr_A\otimes\ve_{|z|})}\otimes\pa{\mP_B^{(j)}\ve}}\\
        =& \pa{\sum_{j=1}^K x_j\mP_{Az}^{(j)}\otimes \mP_B^{(j)}}^{t-1}\pa{\sum_{j=1}^K x_j  \mP_{Az}^{(j)}(\vr_A\otimes\ve_{|z|})}\otimes\ve\\
        =&\ldots = \pa{\pa{\sum_{j=1}^K x_j\mP_{Az}^{(j)}}^t (\vr_A\otimes\ve_{|z|})}\otimes \ve\,.
\end{align*}

\end{proof}

We then provide the measure of Markov entanglement with exogenous information w.r.t agent-wise total variation distance.
\begin{align}
    E_A(\PAB,\gZ)&\coloneqq\min\frac{1}{2}\norm{\mP^\pi_{ABz}-\sum_{j=1}^Kx_j\mP_{Az}^{(j)}\otimes \mP_B^{(j)}}_{\rm{ATV}_1}\nonumber \\
    &=\min_{\mP_{Az}}\max_{\vs,\va,z} \frac{1}{2}\sum_{s_A^\prime,a_A^\prime,z^\prime} \abs{P^\pi_{ABz}(s_A^\prime,a_A^\prime,z^\prime\mid \vs,\va,z) - P_{Az}(s_A^\prime,a_A^\prime,z^\prime\mid s_A,a_A,z)}\,. 
\end{align}

Similar to Theorem~\ref{thm: two-agent atv}, we can connect this measure of Markov entanglement with the value decomposition error. 

\begin{theorem}
    Consider a $N$-agent Markov system $\gM_{1:N}$. Given any policy $\pi\colon \gS \to \Delta(\gA)$ with the measure of Markov entanglement $E_i(\PN,\gZ)$ w.r.t the agent-wise total variation distance, it holds for any agent $i$,
    \[\norm{\mP^\pi_{iz}-\sum_{j=1}^Kx_j\mP^{(j)}_{iz} }_{\infty}\leq 2E_i(\PN,\gZ)\,.\]
    Furthermore, the decomposition error is entry-wise bounded by the measure of Markov entanglement,
    \[\norm{ Q^\pi_{1:N}(\vs,\va,z)-\sum_{i=1}^N Q^\pi_{iz}(s_i,a_i,z) }_\infty \leq \frac{4\gamma \pa{\sum_{i=1}^N E_i(\mP^\pi_{1:N},\gZ)r_{\max}^i}}{(1-\gamma)^2}\,.\]
\end{theorem}

In practice, exogenous information is often discussed in the context of (weakly-)coupled MDPs, where each agent independent evolves by $P_i(s_{i+1}\mid s_i,a_i,z)$. Interestingly, we can derive a similar result to Proposition~\ref{prop: entanglement of policy} that shaves off the transition in entanglement analysis.
\begin{proposition}\label{prop: ME_order}
Consider a $N$-agent Weakly Coupled Markov system $\gM_{1:N}$. Given any policy $\pi\colon \gS \to \Delta(\gA)$ and its measure of Markov entanglement $E_i(\PN,\gZ)$ w.r.t the $\mu^\pi_{1:N}$-weighted agent-wise total variation distance, it holds
\begin{equation}\label{eq: ME_order}
E_i(\PN,\gZ)\leq \frac{1}{2}\sup_i\sum_{s_{1:N},z} \mu^\pi(s_{1:N},z)\sum_{a_i}\abs{\pi(a_i\mid s_{1:N},z)-\pi^\prime(a_i\mid s_i,z)}\,,
\end{equation}
for any policies $\pi^\prime$.
\end{proposition}

\begin{proof}
    \begin{align*}
        E_A(\pi,\gZ)&=\frac{1}{2}\sum_{\vs,\va,z} \mu(\vs,\va,z)\sum_{s_A^\prime,a_A^\prime,z^\prime} \abs{P^\pi_{ABz}(s_A^\prime,a_A^\prime,z^\prime\mid \vs,\va,z) - P_{Az}(s_A^\prime,a_A^\prime,z^\prime\mid s_A,a_A,z)}\\
        &=\frac{1}{2}\sum_{\vs,\va,z} \mu(\vs,\va,z)\sum_{s_A^\prime,a_A^\prime,z^\prime} \abs{\sum_{s_B^\prime}P^\pi_{ABz}(\vs^\prime,a_A,z^\prime\mid \vs,\va,z) - P_{Az}(s_A^\prime,z^\prime\mid s_A,a_A,z)\pi^\prime(a_A^\prime\mid s_A^\prime,z^\prime)}\\
        &=\frac{1}{2}\sum_{\vs,\va,z} \mu(\vs,\va,z)\sum_{s_A^\prime,a_A^\prime,z^\prime} \abs{\sum_{s_B^\prime}P^\pi_{ABz}(\vs^\prime,a_A,z^\prime\mid \vs,\va,z) - \sum_{s_B^\prime}P(\vs^\prime,z^\prime\mid \vs,\va,z)\pi^\prime(a_A^\prime\mid s_A^\prime,z^\prime)}\\
        &=\frac{1}{2}\sum_{\vs,\va,z} \mu(\vs,\va,z)\sum_{s_A^\prime,a_A^\prime,z^\prime} \abs{ \sum_{s_B^\prime}P(\vs^\prime,z^\prime\mid \vs,\va,z)\pa{\pi(a_A^\prime\mid \vs^\prime,z^\prime)-\pi^\prime(a_A^\prime\mid s_A^\prime,z^\prime)}}\\
        &\leq \frac{1}{2}\sum_{\vs,\va,z}\mu(\vs,\va,z)\sum_{\vs^\prime,z^\prime} P(\vs^\prime,z^\prime\mid \vs,\va,z) \sum_{a_A^\prime} \abs{\pi(a_A^\prime\mid \vs^\prime,z^\prime)-\pi^\prime(a_A^\prime\mid s_A^\prime,z^\prime)}\\
        &=\frac{1}{2} \sum_{\vs^\prime,z^\prime} \mu(\vs^\prime,z^\prime) \sum_{a_A^\prime} \abs{\pi(a_A^\prime\mid \vs^\prime,z^\prime)-\pi^\prime(a_A^\prime\mid s_A^\prime,z^\prime)}\,.
    \end{align*}
\end{proof}
\subsection{Factored MDPs}
Another common class of multi-agent MDPs is Factored MDPs (FMDPs, \citealt{carlos01multiagent, carlos02efficient, ian14near}), which explicitly model the structured dependencies in state transitions. For instance, in a server cluster, the state transition of each server depends only on its neighboring servers. Formally, we define

\begin{definition}[Factored MDPs]
    An $N$-agent MDP $\gM_{1:N} (\gS, \gA, \mP , \vr_{1:N} , \gamma)$ is a factored MDP if each agent $i$ has neighbor set $Z_i\in[N]$ such that its transition is affected by all its neighbors, i.e. $P(s_i^\prime \mid \vs,\va)=P(s_i^\prime \mid s_{Z_i},a_{Z_i})$. 
\end{definition}
The neighbor set $|Z_i|$ is often assumed to be much smaller compared to the number of agents $N$. This helps to encode exponentially large system very compactly. We show this idea can also be captured in Markov entanglement. Consider the measure of Markov entanglement w.r.t ATV distance in~\eqref{eq: degree of independent},
\begin{align*}
    E_A(\PAB)
    &= \min_{\mP_A} \max_{(\vs,\va)\in\gS\times\gA}D_{\rm{TV}}\Big(\PAB(\cdot,\cdot\mid \vs,\va), \mP_A(\cdot,\cdot\mid s_A,a_A)\Big)\\
    &= \min_{\mP_A} \max_{(\vs,\va)\in\gS\times\gA}D_{\rm{TV}}\Big(\PAB(\cdot,\cdot\mid s_{Z_A},a_{Z_A}), \mP_A(\cdot,\cdot\mid s_A,a_A)\Big)\,.
\end{align*}
Thus we conclude the agent-wise Markov entanglement will only depend on its neighbor set.
\subsection{Fully Cooperative Markov Games}\label{sec: shared reward}
In fully cooperative settings, only a global reward will be reviewed to all agents. Unlike the modeling in section~\ref{sec: model}, this global reward may not necessarily be decomposed as the summation of local rewards. In this case, we propose meta algorithm~\ref{alg: meta_shared} as an extension of meta algorithm~\ref{alg: meta}.

\begin{algorithm}[h]
\caption{Q-value Decomposition with Shared Reward}
\begin{algorithmic}[1]\label{alg: meta_shared}
\REQUIRE{Global policy $\pi$; horizon length $T$.}
\STATE Execute $\pi$ for $T$ epochs and obtain $\gD=\brk{(s_{AB}^t,a_{AB}^t,r_{AB}^t, s_{AB}^{t+1}, a_{AB}^{t+1})}_{t=1}^{T-1}$.
\STATE Each agent $i\in\{A,B\}$ fits $Q^\pi_i$ using local observations $\gD_i=\brk{(s_i^t,a_i^t,r_i,s_{i}^{t+1}, a_i^{t+1})}_{t=1}^{T-1}$ where the local reward $(\vr_A,\vr_B)$ is learned via solving
\[\min_{\vr_A,\vr_B} \sum_{t=1}^T\Big(r_{AB}^t(\vs,\va)- (r_A(s_A^t,a_A^t)+r_B(s_B^t,a_B^t)) \Big)^2\,.\]
\end{algorithmic}
\end{algorithm}

This algorithm follows similar framework of meta algorithm~\ref{alg: meta} and differs at we now learn the closet local reward decomposition from data. When the reward is completely decomposable, meta algorithm~\ref{alg: meta_shared} recovers meta algorithm~\ref{alg: meta}. Thus intuitively, the more accurate we can decompose the global reward, the less decomposition error we have. Formally, we define the measure of reward entanglement 
\begin{equation}\label{eq: reward_entangle}
e(\vr_{AB})\coloneqq \min_{\vr_A,\vr_B} \norm{\vr_{AB}-(\vr_A\otimes \ve+\ve\otimes \vr_B)}_{\muAB}\,.
\end{equation}
This measure characterizes how accurate we can decompose the global reward under stationary distribution. We then obtain an extension of Theorem~\ref{thm: rho-weighted decomp}
\begin{proposition}\label{prop: shared reward}
    Consider a fully cooperative two-agent Markov system $\gM_{AB}$. Given any policy $\pi\colon \gS \to \Delta(\gA)$ with the measure of Markov entanglement $E_A(\PAB),E_B(\PAB)$ w.r.t the $\mu^\pi_{AB}$-weighted agent-wise total variation distance and the measure of reward entanglement $e(\vr_{AB})$, it holds
    \[\Big\lVert Q^\pi_{AB}-\pa{Q^\pi_A\otimes \ve + \ve \otimes Q^\pi_B} \Big\rVert_{\mu^\pi_{AB}} \leq \frac{e(\vr_{AB})}{1-\gamma} + \frac{4\gamma \pa{E_A(\PAB)r_{\max}^A+E_B(\PAB)r_{\max}^B}}{(1-\gamma)^2}\,,\]
    where $r_{\max}^A, r_{\max}^B$ is the bound of optimal solution of \eqref{eq: reward_entangle}.
\end{proposition}
Although Proposition~\ref{prop: shared reward} offers a theoretical guarantee for general two-agent fully cooperative Markov games, its utility is greatest in systems with low reward and transition entanglement. Fully cooperative settings remain inherently challenging--for instance, even the asymptotically optimal Whittle Index may achieve only a $\frac{1}{N}$-approximation ratio for RMABs with global rewards (\citealt{raman2024global}). In practice, most research (\citealt{peter18value, rashid2020monotonic}) relies on sophisticated deep neural networks to learn decompositions in such settings. We thus defer a more refined analysis of fully cooperative scenarios to future work.

\section{Simulation environments}
All simulation code is available at this \href{https://github.com/shuzechen/markov_entanglement}{Github link}.

\subsection{Circulant RMAB}\label{app: simulation}
In this section, we empirically study the value decomposition for index policies. 
\paragraph{Circulant RMAB} A circulant RMAB has four states indexed by $\{0, 1, 2, 3\}$. Transition kernels $P_a = {p(s, 0, s')}_{s, s' \in S}$ for action $a = 0$ and $a = 1$ are given by
\[\mP_0 = 
\begin{pmatrix}
    1/2 &  \ 0  & \ 0  &\ 1/2\\
    1/2 &\ 1/2 &\ 0 &\ 0\\
    0 &\ 1/2 &\ 1/2 &\ 0\\
    0 &\ 0 &\ 1/2 &\ 1/2
\end{pmatrix}, 
\ 
\mP_1 = 
\begin{pmatrix}
    1/2 &  \ 1/2  & \ 0  &\ 0\\
    0 &\ 1/2 &\ 1/2 &\ 0\\
    0 &\ 0 &\ 1/2 &\ 1/2\\
    1/2 &\ 0 &\ 0 &\ 1/2
\end{pmatrix}.
\]
The reward solely depends on the state and is unaffected by the action: \[r(0, a) = -1, \ r(1, a) = 0, \ r(2, a) = 0, \ r(3, a) = 1; \forall a \in \{0, 1\}.\] We set the discount factor to $\gamma = 0.5$ and require $N/5$ arms to be pulled per period. Initially, there are $N / 6$ arms in state $0$, $N / 3$ arms in state $1$ and $N / 2$ arms in state $2$, the same as \cite{zhang2022near}. We then test an index policy with priority: state $2>$ state $1 >$ state $0 >$ state $3$.

\paragraph{Monte-Carlo estimation of Markov entanglement}
For each RMAB instance, we simulate a trajectory of length $T=6N$ and collect data for the later $5N$ epochs. Notice RMAB is a special instance of WCMDP, we thus apply \eqref{eq: monte-carlo-wcmdp}.
\begin{align}
    E_i(\mP_{1:N}^\pi)\approx\frac{1}{2}\min_{\pi^\prime}\frac{1}{T}\sum_{t=1}^T\sum_{a_i}\abs{\pi(a_i\mid \vs)-\pi^\prime(a_i\mid s_i)}\label{eq: appendix_monte_carlo}
\end{align}
Notice \eqref{eq: appendix_monte_carlo} is \emph{convex} for $\pi^\prime$ and $\pi^\prime$ only takes support of size $|S||A|=8$. we thus apply efficient convex optimization solvers. We replicate this experiment for $10$ independent runs to obtain the mean estimation and standard error in the left panel of Figure~\ref{fig: RMAB_Exp_main_text}.

\paragraph{Learning local Q-values}
For each RMAB instance, we simulate a trajectory of length $T=6N$, reserving the later $T=5N$ epochs as the training phase for each agent to fit local Q-value functions. During testing, we estimate the $\mu$-weighted decomposition error using $50$ simulations sampled from the stationary distribution.

The ground-truth $Q^\pi_{1:N}$ is approximated via Monte Carlo learning (\citealt{Sutton18RL}), with each estimate derived from $30$-step simulations averaged over $3N$ independent runs. Error bars represent the standard error for both Monte Carlo estimates and $\mu$-weighted decomposition errors.

In addition to $\mu$-weighted error, we introduce a concept of relative error, defined as $\norm{Q^\pi_{1:N}(\vs,\va)-\sum_{i=1}^NQ^\pi_i(s_i,a_i)}_{\mu^\pi_{1:N}}/\norm{Q^\pi_{1:N}}_{\mu^\pi_{1:N}}$. This relative error reflects the approximate ratio of our value decomposition. We present our simulation results in Figure~\ref{fig: RMAB_Exp}.

\begin{figure}[h]
\centering
\begin{tabular}{l l}
        \includegraphics[height=.35\textwidth]{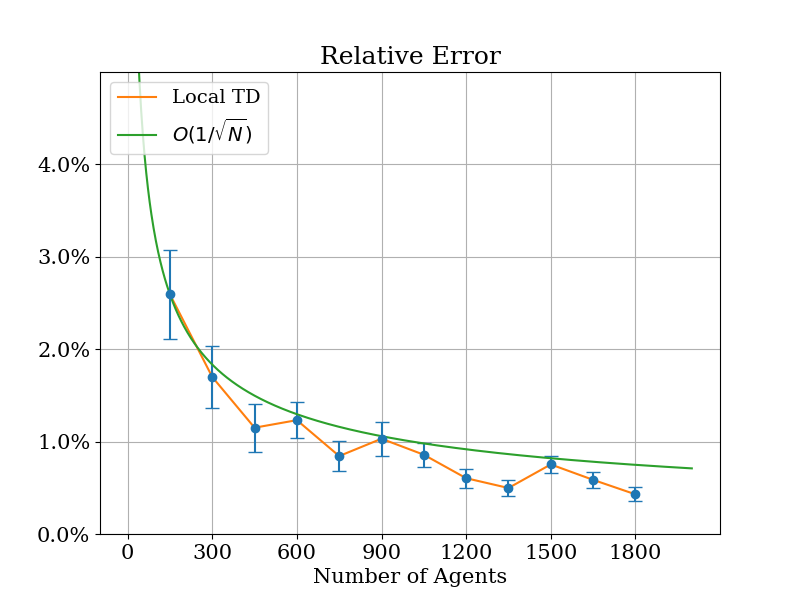}
        &
        \includegraphics[height=.35\textwidth]{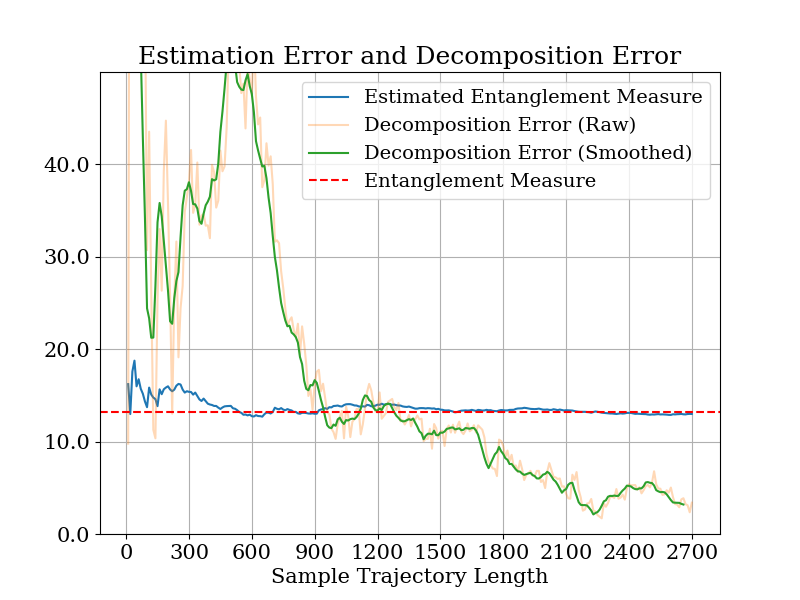}
    \end{tabular}
    \caption{Value Decomposition error in circulant RMAB under an index policy. \emph{Left:} Relative error, $\norm{\textrm{decomposition error}}_\mu/\norm{Q^\pi_{1:N}}_\mu$. \emph{Right: }Different errors in RMAB with $900$ agents: empirical estimation of Markov entanglement (blue); $\mu^\pi_{1:N}$-weighted decomposition error (green); the true measure of Markov estimated with $T=10N$ samples (red dashed line). }
    \label{fig: RMAB_Exp}
\end{figure}
It immediately follows that the relative error decays at rate $\gO(1/\sqrt{N})$ and we notice the relative error is no larger than 3\% over all data points. 

\paragraph{Sample Complexity and Computation} While each RMAB instance has an exponentially large state space $|S|^N$, we show that our empirical estimation of Markov entanglement—along with the decomposition error—converges quickly with $T=5N$. Specifically, we illustrate these errors for an RMAB instance with $900$ agents in the right panel of Figure~\ref{fig: RMAB_Exp}. We see that the empirical estimation of Markov entanglement converges in $T<N$ samples, demonstrating its efficiency.

\subsection{A Ride-hailing Simulator}\label{app: exp_RH}
In this section, we empirically study the Markov entanglement and value decomposition for a ride-hailing simulator.

\paragraph{Ride-hailing Simulator} We collect 810,000 trip records from NYC yellow cabs spanning January through December 2024. The city is partitioned into 268 neighborhood zones, with each trip record containing the pickup and destination zones (see Figure~\ref{fig: taxi_xone}). To simplify the state space, we aggregate geographically proximate zones within Manhattan into $14$ consolidated zones and treat all zones outside Manhattan as a single $15$-th zone, yielding a local state space of size 15. We then fit the order distribution over these aggregated zones using the trip record data, with rewards defined as the average tolls paid for trips between zone pairs.

\begin{figure}[h]
    \centering
    \includegraphics[width=0.4\linewidth]{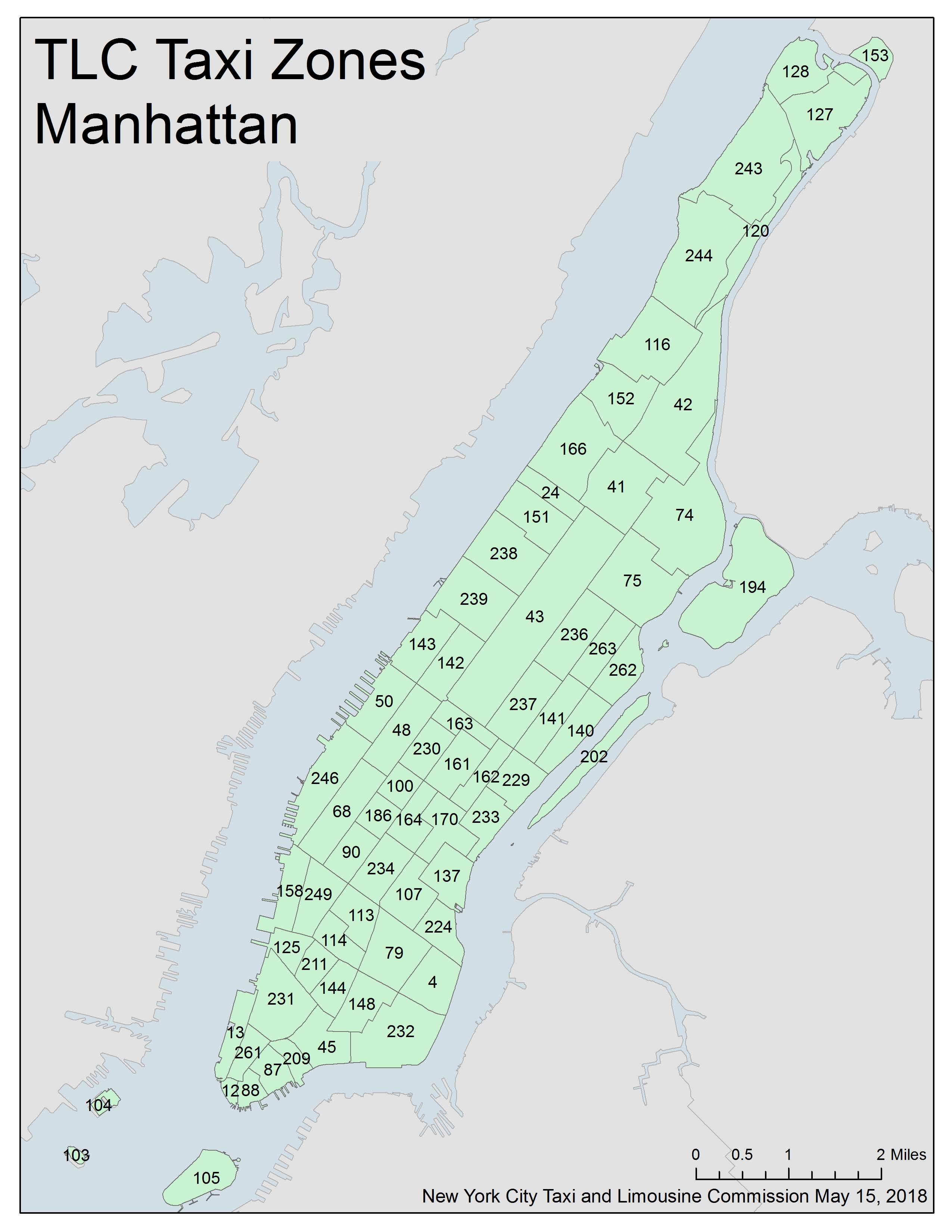}
    \caption{Taxi zone map of Manhattan \cite{NYCTLCData}}
    \label{fig: taxi_xone}
\end{figure}

At each timestep, we sample $0.1N$ orders from this fitted distribution. Each driver's local action is represented as a $(0.1N+1)$-dimensional binary vector with a single element equal to $1$, indicating either acceptance of a specific order or remaining idle. The dispatching mechanism implements a matching policy that minimizes total pickup distance, estimated using actual trip distances from the data. For tractability, our simulator does not incorporate travel delays; we leave this extension for future work. Finally, each idle driver may relocate to a neighboring zone with probability $0.05$.

\paragraph{Monte-Carlo estimation of Markov entanglement}
We apply \eqref{eq: ME_order} in Proposition~\ref{prop: ME_order},
\begin{align}
    E_i(\PN,\gZ)&\leq \frac{1}{2}\sup_i\sum_{s_{1:N},z} \mu^\pi(s_{1:N},z)\sum_{a_i}\abs{\pi(a_i\mid s_{1:N},z)-\pi^\prime(a_i\mid s_i,z)}\label{eq: me-wcmdp-order}\\
    &\approx \min_{\pi^\prime}\frac{1}{2T}\sum_{t=1}^T \sum_{a_i}\Big|\pi(a_i\mid \vs^t,z^t)-\pi^\prime(a_i\mid s_i^t,z^t)\Big|\,.\label{eq: monte-carlo-wcmdp-order}
\end{align}
We first notice that $E_i(\PN,\gZ)$ is identical across all drivers due to their homogeneity in our simulator. This property enables us to aggregate estimates from individual drivers to update a centralized entanglement estimator. Specifically, 
\begin{equation}\label{eq: central_estimate}
E(\PN,\gZ)\approx \min_{\pi^\prime}\frac{1}{2T}\sum_{t=1}^T \sum_{k=1}^d \frac{|N(s^t_k)|}{N} \sum_{a_i}\Big|\pi(a_i\mid \vs^t,z^t)-\pi^\prime(a_i\mid s_k,z^t)\Big|\,,
\end{equation}
where $d$ is the local state space size and $N(s^t_k)$ is the number of agents at state $s_k$ at timestep $t$. \eqref{eq: central_estimate} can also be verified as taking average of all estimated $E_i(\PN,\gZ)$. This centralized estimation is $N$ times more sample-efficient thanks to the homogeneity of agents.

However, \eqref{eq: central_estimate} still suffers from the curse of dimensionality due to the high-dimensional exogenous order space $\gZ$. Recall we sample $0.1 N$ orders at each timestep, yielding $d^{0.2N}$ possible combinations of orders in $\gZ$. To address this issue, we take advantage of the exogeneity of $\gZ$ and transform \eqref{eq: me-wcmdp-order},
\begin{align*}
    E_i(\PN,\gZ)&\leq \frac{1}{2}\sum_{s_{1:N},z} \mu^\pi(s_{1:N},z)\sum_{a_i}\abs{\pi(a_i\mid s_{1:N},z)-\pi^\prime(a_i\mid s_i,z)}\\
    &= \frac{1}{2}\sum_{z} P(z) \sum_{s_{1:N},z} \mu^\pi(s_{1:N})\sum_{a_i}\abs{\pi(a_i\mid s_{1:N},z)-\pi^\prime(a_i\mid s_i,z)}\\
    &\approx \frac{1}{T}\sum_{l=1}^T\min_{\pi^\prime(\cdot|s,z^l)}\frac{1}{2T}\sum_{t=1}^T \sum_{a_i}\Big|\pi(a_i\mid \vs^t,z^l)-\pi^\prime(a_i\mid s_i^t,z^l)\Big|\,.
\end{align*}
Combined with \eqref{eq: central_estimate} we obtain our final estimator of Markov entanglement measure for the ride-hailing setting.

\paragraph{Learning local Q-values} In Appendix~\ref{app: exog_order}, we extend our Markov entanglement theory with local Q-value defined as $Q^\pi_i(s_i,a_i,z)$. As mentioned above, $z$ has exponentially large support $d^{0.2N}$, which renders original tabular TD learning untractable. To address this issue, we again take advantage of the ride-hailing structure. Notice that we can define local value function using the following Bellman equation
\begin{equation}\label{eq: RH-TD}
Q^\pi_i(s_i,a_i,z)=r_i(s_i,a_i,z)+\sum_{s^\prime_i} p(s^\prime_i|s_i,a_i,z) V^\pi_i(s_i^\prime)\,.
\end{equation}
The key idea is that the transition $p(s^\prime_i|s_i,a_i,z)$ turns out to be very easy to estimate for the ride-hailing setting. When $a_i$ corresponds to taking certain order, then the transition is fixed since the driver will transit to the destination zone of the order; when $a_i$ corresponds to stay idle, $p(s^\prime_i|s_i,a_i,z)$ refers to the relocation probability that does not depend on $z$. This idea then reduces learning local Q-value to learning local value functions $V^\pi_i$, which has only constant support $d$. Furthermore, we notice that since drivers are homogeneous. Thus we can aggregate their local TD updates to learn a central local value function, improving sample efficiency by a factor of $N$.

Finally, we note \eqref{eq: RH-TD} is exactly how the real-world ride-hailing platform Lyft conducts value decomposition. Our local value function corresponds to what is called the Online Supply Values (OSV) in~\cite{han22real}.


\end{appendices}



\end{document}